\documentclass[11pt]{article}

\usepackage[margin=1in]{geometry}

\usepackage[utf8]{inputenc}
\usepackage[english]{babel}
\usepackage{amsmath,amssymb,epsfig,amsthm}
\usepackage{graphicx}
\usepackage{comment}
\usepackage{array}
\usepackage{algorithm}
\usepackage{enumerate}
\usepackage{xcolor}
\usepackage{algpseudocode}
\usepackage{xurl}

\newtheorem{theorem}{Theorem}

\newtheorem{lemma}{Lemma}

\newtheorem{proposition}{Proposition}
\newtheorem{remark}{Remark}

\numberwithin{equation}{section}

\DeclareMathOperator{\Var}{Var}
\DeclareMathOperator{\Ima}{Im}
\DeclareMathOperator{\maskkpod}{\Omega_U}

\newcommand{\calB}{\ensuremath{\mathcal{B}}}
\newcommand{\calC}{\ensuremath{\mathcal{C}}}

\newcommand{\calG}{\ensuremath{\mathcal{G}}}

\newcommand{\calS}{\ensuremath{\mathcal{S}}}
\newcommand{\calM}{\ensuremath{\mathcal{M}}}
\newcommand{\calN}{\ensuremath{\mathcal{N}}}
\newcommand{\calT}{\ensuremath{\mathcal{T}}}

\newcommand{\calK}{\ensuremath{\mathcal{K}}}

\newcommand{\norm}[1]{\left\|{#1}\right\|}
\newcommand{\abs}[1]{\left|{#1}\right|}

\newcommand{\set}[1]{\left\{{#1}\right\}}

\newcommand{\expec}{\ensuremath{\mathbb{E}}}
\newcommand{\matR}{\ensuremath{\mathbb{R}}}

\newcommand{\argmax}[1]{\underset{#1}{\operatorname{argmax}}}
\newcommand{\argmin}[1]{\underset{#1}{\operatorname{argmin}}}

\newcommand{\prob}{\ensuremath{\mathbb{P}}}

\newcommand{\gb}[1]{\textcolor{black}{#1}}

\newcommand{\indic}{\ensuremath{\mathbf{1}}} 

\newcommand{\kpod}{\texttt{k-pod}}
\newcommand{\sumadjz}{\texttt{sumAdj0}}
\newcommand{\sumadji}{\texttt{sumAdjIter}}
\newcommand{\olmfm}{\texttt{OLMFm}}
\newcommand{\lapla}{\texttt{Laplacian}}
\newcommand{\aggrk}{\texttt{AggrKern}}

\usepackage[round]{natbib}
\begin{document}

%

%

\title{Clustering multilayer graphs with missing nodes}

\author{ Guillaume Braun\footnotemark[1]\\ \texttt{guillaume.braun@inria.fr }  \and Hemant Tyagi\footnotemark[1]\\ \texttt{hemant.tyagi@inria.fr} \and Christophe Biernacki\footnotemark[1]\\ \texttt{christophe.biernacki@inria.fr}   }

\renewcommand{\thefootnote}{\fnsymbol{footnote}}
\footnotetext[1]{Inria, Univ. Lille, CNRS, UMR 8524 - Laboratoire Paul Painlev\'{e}, F-59000 }

\renewcommand{\thefootnote}{\arabic{footnote}}

\maketitle

\begin{abstract}Relationship between agents can be conveniently represented by graphs. When these relationships have different modalities, they are better modelled by multilayer graphs where each layer is associated with one modality. Such graphs arise naturally in many contexts including biological and social networks. Clustering is a fundamental problem in network analysis where the goal is to regroup nodes with similar connectivity profiles. In the past decade, various clustering methods have been extended from the unilayer setting to multilayer graphs in order to incorporate the information provided by each layer. While most existing works assume – rather restrictively - that all layers share the same set of nodes, we propose a new framework that allows for layers to be defined on different sets of nodes. In particular, the nodes not recorded in a layer are treated as missing. Within this paradigm, we investigate several generalizations of well-known clustering methods in the complete setting to the incomplete one and prove some 
consistency results under the Multi-Layer Stochastic Block Model assumption. Our theoretical results are complemented by thorough numerical comparisons between our proposed algorithms on synthetic data, and also on real datasets, thus highlighting the promising behaviour of our methods in various settings. 
\end{abstract}

\section{Introduction}

Graphs are a powerful tool to represent relationships between agents. Due to applications in a wide array of fields including biology, sociology, ecology and economics (see for e.g., \cite{brain, airoldi2014,kivela_eco, Kim2015}), the analysis of networks has received significant interest over the last two decades. One fundamental problem of network analysis is \emph{clustering} which involves detecting communities by regrouping nodes having similar connectivity properties. Numerous clustering algorithms have been developed over the years based on different approaches such as modularity maximization, maximum likelihood, random walks, semi-definite programming and spectral clustering (see for instance the survey articles by \cite{fortunato2010} and \cite{AbbeSBM}).


 Often, relationships are better understood through different modalities. These multiple aspects of relationships can be represented by a multilayer graph where each layer is a graph representing the interactions between agents for one modality. For e.g., social interaction between a set of people can be recorded via email exchanges, phone calls, professional links, and so on. Each level of interaction can be encoded into a simple graph and the collection of these graphs leads to a multilayer representation. Another important example of a multilayer graph is given by a  time-varying network where each view of the network at a given time corresponds to a different layer.

 Over the last decade, many methods have been proposed for clustering multilayer graphs such as those based on matrix factorization, spectral methods, maximisation of a modularity function or probability model-based approaches; see \cite{Kim2015} for a survey. Consistency results for the recovery of the partition under a stochastic generative model have also been shown for some algorithms, see for example \cite{paul2020}, \cite{pensky2019}, \cite{lei2020tail} and \cite{Bhattacharyya2018SpectralCF}. 

Most existing approaches assume that all the layers share the same set of nodes. In practice, however, data are often incomplete; in particular, the set of observed nodes can clearly vary across layers. For example, in social networks evolving over time, the set of nodes can change due to people leaving/joining the network. This is the setting considered in the present paper.

\subsection{Related work}


\paragraph{Clustering on multi-layer graphs. } As noted by \cite{paul2020}, clustering strategies for multilayer graphs can be roughly categorized into three groups: {\bf early fusion} methods where all views are aggregated and then clustering is performed, {\bf intermediate fusion} methods where the algorithm finds a factor common to all the views, and {\bf final aggregation} methods where each individual view is processed separately and a consensus partition is formed.  
In the complete setting, different algorithms have been proven to be consistent under a multilayer stochastic block model assumption (see Section \ref{subsec:mlsbm}). Among them are spectral clustering on the sum of adjacency matrices (e.g., \cite{Bhattacharyya2018SpectralCF,paul2020}) or on the sum of squared adjacency matrices with bias correction (e.g., \cite{lei2020tail,bhattacharyya2020general}), orthogonal linked matrix factorization (e.g., \cite{paul2020}), and co-regularized spectral clustering (e.g., \cite{paul2020}). Existing misclustering bounds for these methods are gathered in the supplementary material. 
\gb{
\paragraph{Incomplete Multi-View Clustering (IMVC).}Recently a similar problem has been addressed in the  context of IMVC, see for example \cite{imvc}, \cite{hu2019onepass} and references therein. To the best of our knowledge, no consistency results for the recovery of the ground truth clustering structure are shown in this line of work. Algorithms designed for the IMVC framework cannot be directly applied to our setting since they apply to a collection of feature vectors. However they could possibly be adapted, in a non trivial manner, to our framework. For example, in the complete setting, the OMVC method proposed by \cite{hu2019onepass} can be considered as a variant of the OLMF estimator proposed by \cite{paul2020}  where the optimization problem is modified in order to take into account the symmetry of the inputs. Similarly, if there were no missing views, the algorithm proposed by \cite{imvc} resembles a variant of the co-regularized spectral clustering method of \cite{paul2020} for clustering multilayer graphs. We leave the adaptation of the algorithm proposed by \cite{imvc} to our setting for future work. }



\subsection{Contributions} 
We consider the problem of clustering multilayer graphs with missing nodes under a Multi-Layer Stochastic Block Model (MLSBM)  described in Section \ref{sec:probsetup}. Our contributions are as follows.
\begin{itemize}
    \item  In Section \ref{subsec:kpod_fin_agg} we propose a final aggregation method based on a variant of $k$-means for incomplete data (Algorithm \ref{alg:k-pod}), and derive a  bound for the misclustering rate.
    \item  Section \ref{sec:early_fusion}  extends a popular early fusion method -- based on spectral clustering applied to the sum of adjacency matrices -- to the missing nodes setting. Section \ref{subsec:early_fusion_impute_zero} studies this by imputing the missing entries with zeros (Algorithm \ref{alg:sum_of_adj0}), and contains an upper bound for the misclutering rate. Section \ref{subsec:early_fusion_impute_iter} proposes an alternative method (Algorithm \ref{alg:sum_of_adj_c}) wherein the missing entries are imputed iteratively. This method is shown to perform well in our experiments.
    \item  Section  \ref{subsec:olmf_missing} proposes an extension of an intermediate fusion method -- namely the Orthogonal Linked Matrix Factorization (OLMF) method studied by \cite{paul2020} -- to the missing nodes setting.
    \item In Section \ref{sec:num_exps} we empirically evaluate our algorithms on synthetic data, and also on real datasets.
\end{itemize}
%


\subsection{Notations}
The set of integers $\lbrace 1, \ldots, n\rbrace$ will be denoted by $[n]$. For a matrix $M\in \mathbb{R}^{n\times n}$, its Frobenius (resp. operator) norm is denoted by $||M||_F$ (resp.  $||M||$). The notation $M_{i*}$ (resp. $M_{*j}$)  denotes the $i$-th row (resp. $j$-th column) of $M$. For any subset $J$ of $[n]$ and symmetric matrix $M\in \mathbb{R}^{n\times n}$, $M_J \in \mathbb{R}^{|J|\times |J|}$  denotes the square submatrix of $M$ obtained by deleting rows and columns whose index doesn't belong to $J$. For a non symmetric matrix $Z\in \mathbb{R}^{n\times K}$, $Z_{J}$ denotes the submatrix of $Z$ obtained by deleting rows whose index doesn't belong to $J$.
Sometimes, it will also be convenient to consider $A_J$ (resp. $Z_J$) as a $n\times n$ (resp. $n\times K$) matrix where the rows and columns (resp. only the rows) whose index doesn't belong to $J$ are filled with zeros; this will be clear from the context.
$I_n$ denotes the identity matrix of size $n$. Constants will be denoted by the letters $c$ and $C$, eventually indexed by a number to avoid confusion. Within proofs the values of constants can change from line to line whereas they are denoted with the same letter for simplicity.
\section{Problem setup} \label{sec:probsetup}
A  multilayer graph is a sequence of graphs $\calG = (\calG ^{(1)}, \ldots , \calG ^{(L)})$. If all the graphs are defined on the same set of nodes $\calN $ indexed by $[n]$, then $\calG$ is said to be pillar. Throughout, we will assume that for all $l\leq L$ each graph $\calG ^{(l)}$ is undirected and has no self-loop. This implies that its associated adjacency matrix $A^{(l)} \in \set{0,1}^{n \times n}$ is symmetric with $A^{(l)}_{ii}=0$ for all $i$. 

Given $\calG$ as input, our goal is to recover a partition of $\calN$ into $K$ disjoint sets (or communities), so that nodes belonging to the same community share a similar connectivity profile. To make the setup more precise, we will study this problem in the setting where $\calG$ is generated via an underlying (unknown) stochastic model, with a latent community structure. This model is a common extension of the well-studied stochastic block model (SBM) for the unilayer case which we now describe.

\subsection{Stochastic Block Model (SBM)}
The stochastic block model (SBM) -- first proposed in \cite{HOLLAND1983109} -- is a simple yet popular stochastic generative model for unilayer graphs which captures the community structures of networks often observed in the real world. 
A SBM with the set of nodes $\calN$ and $K$ communities $\calC_1, \ldots ,\calC_K$ forming a partition of $\calN$ is parameterized as follows. \begin{itemize}
    \item There is a membership matrix $Z \in \calM_{n,K}$ where $\calM_{n,K}$ denotes the class of membership matrices. Here, $Z_{ik}=1$ if node $i$ belongs to $\calC_k$, $0$ otherwise. Each membership matrix $Z$ can be associated bijectively with a function $z:[n]\to [K]$ such that $z(i)=k$ where $k$ is the unique column index satisfying $Z_{ik}=1$.
    \item There is a full-rank, symmetric, connectivity matrix of probabilities $$\Pi=(\pi_{k k'})_{k, k' \in [K]} \in [0,1]^{K \times K}.$$
\end{itemize}
Let us denote $P=(p_{ij})_{i,j \in [n] }:=Z\Pi Z^T$. A graph $\calG$ is distributed according to a stochastic block model SBM$(Z, \Pi)$ if the corresponding symmetric adjacency matrix $A$  has zero diagonal entries and 
\[ 
A_{ij}  \overset{\text{ind.}}{\sim} \mathcal{B}(p_{ij}), \quad 1 \leq i < j \leq n,
\] 
where $\calB(p)$ denotes a Bernoulli distribution with parameter $p$. Hence the probability that two nodes are connected depends only on the community memberships of these two nodes. 

Let us denote by $n_k$ the size of the community $\calC_k$, $n_{min}$ (resp. $n_{max}$) to be the size of the smallest (resp. largest) community, and $\beta = \frac{n_{max}}{n_{min}}$. The communities are said to be balanced if they all have the same size (equivalently, $\beta=1$). The communities are approximately balanced if $\beta=O(1)$. The maximum value of the connectivity parameter is denoted by $p_{max} := \max_{i,j} p_{ij}$ and can be interpreted as the sparsity level (depending on $n$).

The misclustering rate associated to an estimated membership matrix $\hat{Z}$ is measured by \[ r(\hat{Z},Z)=r(\hat{z},z)=\frac{1}{n}\min _{\sigma \in \mathfrak{S}}\sum_i \indic_{\lbrace \hat{z}(i)\neq \sigma(z(i))\rbrace},\] where $\mathfrak{S}$ denotes the set of permutations on $[K]$. A clustering algorithm is said to be strongly consistent -- or achieving exact recovery -- if $r(\hat{Z},Z)=0$ with probability $1-o(1)$ as $n$ tends to infinity. It is said to be weakly consistent -- or achieving almost exact recovery -- if $\prob(r(\hat{Z},Z)=o(1))=1-o(1)$ as $n$ tends to infinity. A more complete overview of the different types of consistency and the sparsity regimes where they occur can be found in \cite{AbbeSBM}.

\subsection{Multilayer Stochastic Block Model (MLSBM)} \label{subsec:mlsbm}
We now describe the multilayer stochastic block model (MLSBM), which is a common extension of the SBM to the setting of multilayer graphs (see for e.g.,  \cite{paul2020,Bhattacharyya2018SpectralCF, jinLei2019}). The MLSBM is parametrized by the number of layers $L$, a common block membership matrix $Z\in \calM_{n,K}$, and connectivity matrices $\Pi^{(1)}, \ldots, \Pi^{(L)} \in [0,1]^{K \times K}$. 

\paragraph*{}Similar to the unilayer case, let us denote $P^{(l)}=Z\Pi^{(l)}Z^T$ for $l=1,\ldots, L$. A multilayer graph $\calG$ is distributed according to the model MLSBM$(Z, \Pi^{(1)}, \ldots, \Pi^{(L)})$ if the adjacency matrix $A^{(l)}$ of each layer is distributed according to a SBM$(Z,\Pi^{(l)})$ for $l=1,\ldots, L$. Hence, while the probability that two nodes are connected can vary across layers, the block membership of each node remains unchanged. As in the unilayer case we can define the quantities 
%
$p_{max}^{(l)}=\max_{i,j}p_{ij}^{(l)}, \quad  p_{max}=\max_l p_{max}^{(l)}.$
%
%
%
\subsection{Missing nodes} \label{subsec:missing_node_MLSBM}
The assumption that all the layers share the same set of nodes is quite restrictive since real world multilayer networks are often `non-pillar'. 
We propose to deal with such networks by considering nodes present in some layers but not in others as missing.
Let $w_i^{(l)}$ be a binary variable that records the presence of node $i$ in the layer $l$ where $w_i^{(l)}=1$ if node $i$ is observed in layer $l$ and $0$ otherwise. Denoting $w^{(l)} = (w^{(l)}_1, \ldots ,w^{(l)}_n)^T$, let $\Omega^{(l)}=w^{(l)}(w^{(l)})^T$ be the mask matrices and $\Tilde{A}^{(l)}=A^{(l)}\odot \Omega^{(l)}$ for $l\leq L$ where $\odot$ is the usual Hadamard product. Let $J_l$ denote the set of non-missing nodes in  layer $l$ with $n_{J_l} = \abs{J_l}$.
By a slight abuse of notation we will denote by $A_{J_l}$ the matrix $A_{J_l}^{(l)}$. The number of observed nodes in $\calC_k$ will also be denoted by $n_{J_l, k}$. 
%
Throughout, we assume that the missing nodes are generated as $w_i^{(l)}\overset{\text{ind.}}{\sim} \mathcal{B}(\rho)$ for $i=1,\ldots, n.$

\section{Final aggregation methods} \label{sec:fin_agg_methods}
A natural way to extend unilayer graph clustering to the multilayer setting is to analyze each layer separately and then find a consensus partition -- such approaches are referred to as final aggregation methods. For example, one can apply any clustering method on each individual layer, take one layer's labels as a reference,  find for each remaining layer the permutation of its labels that maximizes the agreement with the reference layer, and then define a consensus community by majority voting as discussed in \cite{airoldi2014}. 
There exist alternative ways to avoid the cumbersome issue of label switching ambiguity such as the `aggregate spectral kernel' considered in \cite{paul2020}. 
Such methods rely on the quality of each individual layer and are often empirically outperformed by other methods as shown in \cite{paul2020,airoldi2014}. 

Final aggregation methods are still relevant in the missing nodes context. Indeed, if we have exact recovery for each layer, and if for all $k$ there is at least one common node between two layers belonging to $\calC_k$, then we can easily reconstruct the whole partition even when the set of common nodes is very small. Hence such methods can be considered as baseline methods.

\subsection{A method based on a variant of $k$-means for incomplete data} \label{subsec:kpod_fin_agg}
We now propose a final aggregation method for clustering multilayer graphs in the incomplete setting; it avoids the aforementioned label switching problem. 

For each layer $l$, we can compute the matrix $\hat{U}_{J_l}$ of size $|J_l|\times K$ corresponding to the eigenvectors associated with the top $K$ eigenvalues (in absolute value) of $A_{J_l} \in \mathbb{R}^{\abs{J_l} \times \abs{J_l}}$. 
The matrix $\hat{U}_{J_l}$ can be transformed to a matrix $\hat{U}^{(l)}$ of size $n\times K$ by completing with 0 
the rows of the nodes that haven't been observed\footnote{It is easy to verify that $\hat{U}^{(l)}$ is also the eigenvector matrix corresponding to the top $K$ eigenvalues (in absolute value) of $A^{(l)}\odot \Omega^{(l)}$.}. Let $\hat{U}$ be the $n\times KL$ matrix  obtained by stacking $\hat{U}^{(l)}$. 

Analogously, let $U_{J_l}$ be the matrix formed by the $K$  eigenvectors corresponding to non-zero eigenvalues of $Z_{J_l}\Pi^{(l)}Z_{J_l}^T$, $U^{(l)}$ be the $n\times K$ matrix obtained from $U_{J_l}$ by filling the rows corresponding to unobserved nodes with the row corresponding to an observed node (belonging to the same community), and
$U$ be the matrix obtained by stacking all the matrices $U^{(l)}$. For each $l$, let $O_l$ be a $K \times K$ orthogonal matrix such that
\[ O_l \in \argmin{O^T O=I_k}||\hat{U}_{J_l}-U_{J_l}O ||_F. \]
As in the unilayer setting, $k$-means could be applied on the rows of $\hat{U}^{(l)}$ in order to recover the community structure for each $l$. But in order to avoid the label switching problem we propose to apply on the rows of $\hat{U}$ a variant of $k$-means described in \cite{kpod} that can handle missing values, see Algorithm \ref{alg:k-pod}. 

Let us describe the principle behind this algorithm. The classical $k$-means problem seeks a partition $Z$ and centroid values (encoded in the matrix $C$) that solves \[ \min_{\substack{Z\in \mathcal{M}_{n,K} \\ C\in \mathbb{R}^{K\times KL}}} ||\hat{U}-ZM ||_F^2 .\] When there are missing values one can instead solve
\begin{equation}\label{final:eq1}
    \min_{\substack{Z\in \mathcal{M}_{n,K} \\ C\in \mathbb{R}^{K\times KL}}} ||(\hat{U}-ZM)\odot \maskkpod ||_F^2
\end{equation}   
where $\maskkpod = (w^{(1)}\otimes \textbf{1}_K \ \cdots \  w^{(L)}\otimes \textbf{1}_K)$ is the $n\times KL$ mask matrix with
$\textbf{1}_K \in \mathbb{R}^{1 \times K}$ denoting the all ones vector. It is a matrix composed of $L$ blocks where the rows of each block are $1$ if the corresponding node is observed and $0$ otherwise.

\begin{algorithm}[!ht]
\caption{$k$-pod clustering}\label{alg:k-pod}
\begin{flushleft}
        \textbf{Input:} The number of communities $K$, the sets $J_l$ and the adjacency matrices $A_{J_l}$.
\end{flushleft}
\begin{algorithmic}[1]
\State Form $\hat{U}^{(l)}$ from $A_{J_l}$ as explained at the beginning of Section 3.1.
\State Form the matrix $\hat{U}$ by stacking the matrices $\hat{U}^{(l)}$. 
\State Initialize the partition $\hat{Z}$ and the centroid matrix $\hat{M}$.

\Repeat
\State  Replace $\hat{U}$ by $\hat{U}\odot \maskkpod + (\hat{Z}\hat{M})\odot (\mathbf{1}\mathbf{1}^T-\maskkpod)$.

\State  Apply $K$-means on the complete matrix $\hat{U}$ and update $\hat{M}$ and $\hat{Z}$.
\Until  convergence.
\end{algorithmic}
 \textbf{Output:} A partition of the nodes $\calN =\cup_{i=1}^K \calC_i$ based on $\hat{Z}$.
\end{algorithm}
\gb{In the worst case, the complexity of the algorithm is $O((L+K)n^2)$. But in practice the layers are often sparse and so the complexity will be much less\footnote{This remark regarding the complexity applies to our other methods as well.}.}
%
%
%
%
\gb{
\begin{theorem}
\label{thm:kpod}
Consider the missing nodes MLSBM in Section \ref{subsec:missing_node_MLSBM}, and suppose that $\rho L\geq 1$, $K L\leq C_0n$, $\rho n_{min}\geq C_1 K^2 \max(\log^2 n, \sqrt{np_{max}})$ and $np_{max}^{(l)}\geq C_2\rho ^{-1}\log n$. 
Let $\lambda_K^{(l)}$ be the $K$-th largest singular value of $\Pi^{(l)}$ and recall that $\beta=n_{max}/n_{min}$. 
If \[ \frac{1}{\rho Ln}\sum_l \frac{p_{max}^{(l)}}{(\lambda_K^{(l)})^2} < (30C_3 \beta^4K^3)^{-1}\] then with probability at least $1-O(n^{-1})$, it holds that the solution $\hat{Z} \in \mathcal{M}_{n,K}$ of \eqref{final:eq1} satisfies \[ r(\hat{Z},Z) \leq C_4\exp(-c'\rho L)+\frac{C_5\beta^3K^2}{\rho Ln}\sum_l\frac{p_{max}^{(l)}}{(\lambda_K^{(l)})^2}.\]
\end{theorem}
}
The proof of all our theoretical results are deferred to the supplementary material.
%
\begin{remark}
The assumption $\rho L\geq 1$ is natural since $\rho L$ corresponds to the expected total number of times a node is observed, and a node needs to be observed at least once in order to be classified.
The condition $\rho n_{min}\geq C_1 K^2 \log^2 n$ ensures that $\rho$ and $n_{min}$ are not too small. If the communities are well-balanced and the parameters $\rho$ and $K$ are fixed independently of $n$, then the previous condition is satisfied for $n$ large enough.
\end{remark}
\begin{remark}
Our analysis assumes that each layer is sufficiently informative, and doesn't use the fact that there is more information contained in the whole set of layers than in individual layers. This is why the bound does not improve when $L$ increases. The obtained upper-bound is unlikely to be optimal since as shown in the experiments, the clustering performance does seem to improve a bit when $L$ increases.
\end{remark}

\section{Early fusion methods: spectral clustering on sum of adjacency matrices} \label{sec:early_fusion}
%
%
Late fusion methods rely heavily on the quality of each layer. However, by  simultaneously using all the information contained in all layers, the clustering performance can be improved in some settings (see the numerical experiments in \cite{paul2020} or \cite{airoldi2014}).
One way to do this is to aggregate the information across layers and then apply a suitable clustering method. This approach will be referred to as an early fusion method.
One simple but popular way to do this is to take the mean of the adjacency matrices (see for e.g., \cite{Bhattacharyya2018SpectralCF, paul2020}). Then, the $k$-means algorithm can be applied to the rows of the $n \times K$ eigenvector matrix associated with the top $K$ eigenvalues (in absolute value) of  $A = L^{-1} \sum_l A^{(l)}$.

\subsection{Imputing missing entries with zeros} \label{subsec:early_fusion_impute_zero}
A natural way to extend the aforementioned approach to the setting of missing nodes is to fill the missing entries with zeros, thus leading to Algorithm \ref{alg:sum_of_adj0}. \gb{The worst-case complexity of the algorithm is $O((L+K)n^2)$.}
%
%
\begin{algorithm}[!ht]
\caption{Sum of adjacency matrices with missing entries filled with zeros}\label{alg:sum_of_adj0}
\begin{flushleft}
        \textbf{Input:} The number of communities $K$, the matrices $A^{(l)}$ and $\Omega^{(l)}$.
\end{flushleft}
\begin{algorithmic}[1]
\State Compute $A=L^{-1}\sum_lA^{(l)}\odot \Omega^{(l)}$.

\State Compute the eigenvectors $u_1, \ldots, u_K$ associated with the $K$ largest eigenvalues of $A$ (ordered in absolute values) and form $U_K = [u_1 \ u_2 \ \cdots \ u_K]$.

\State Apply $K$-means on the rows of $U_K$ to obtain a partition of $\calN$ into $K$ communities.
\end{algorithmic}
\textbf{Output:} A partition of the nodes $\calN =\cup_{i=1}^K \calC_i$.
\end{algorithm}

Let us denote  $\Tilde{A}=\rho^{-2}L^{-1}\sum_l A^{(l)}\odot \Omega^{(l)}$ (clustering on $A$ or $\Tilde{A}$ is equivalent since the two matrices are proportional, but for the analysis it is more convenient to work with $\Tilde{A}$). Clearly $\expec(\Tilde{A})=L^{-1}\sum_l \expec(A^{(l)})$ (since the diagonal entries of $A^{(l)}$ are zero). Denote by $\expec(X|\Omega)$ to be the expectation of $X$ conditionally on $\Omega=(\Omega^{(1)}, \ldots, \Omega^{(L)})$ and 
let $\lambda_K$ denote the $K$th largest singular value of $\expec(\Tilde{A})$.
We have $\expec(\Tilde{A}| \Omega)=\rho^{-2}L^{-1}\sum_l \expec(A^{(l)})\odot \Omega^{(l)}$.
Using the same kind of perturbation arguments and concentration inequalities as in~\cite{lei2015}, we can relate $\tilde{A}$ to  $\expec(\Tilde{A}| \Omega)$ and then use Bernstein inequality to relate $\expec(\Tilde{A}| \Omega)$ with  $\expec(\Tilde{A})$. This leads to the following bound on the misclustering rate.
\gb{
\begin{theorem}
\label{thm:sumadj}
 Under the missing nodes MLSBM in Section \ref{subsec:missing_node_MLSBM}, there exist constants $C_0, C_1 > 0$ such that with probability at least $1-O(n^{-1})$, the solution $\hat{Z} \in \mathcal{M}_{n,K}$ obtained from Algorithm \ref{alg:sum_of_adj0} satisfies
\begin{align*}
&r(\hat{Z},Z)\leq \underbrace{\frac{C_0K}{\rho ^4\lambda_K^2}\left(\frac{np_{max}}{L}+\frac{\log n }{L}\right)}_\text{noise error}+\\
&  \underbrace{C_1K\frac{(\rho^{-2}-1)^2}{\lambda_K^2} \left((np_{max})^2\frac{\log(n)}{L} +
\left(\frac{np_{max} \log n}{L}\right)^2 \right)}_\text{missing data error}.
\end{align*}
\end{theorem}
}
If $L$ is small then the missing data error could be larger than one making the upper bound  trivial. In the best case scenario, we expect that $\lambda_K$ scales as $np_{max}$. So we need at least $C\log n$ layers to get a non trivial upper bound. In order to obtain asymptotic consistency, it is necessary that $L\gg \log n$.
However, experiments show that even when $L$ is small, Algorithm \ref{alg:sum_of_adj0} gives good results as long as the layers are dense enough and the number of missing nodes is not too large.

When $\rho=1$ and $np_{max}\geq \log n$  the upper bound becomes $O((Lnp_{max})^{-1/2})$ thus matching the bound obtained by \cite{Bhattacharyya2018SpectralCF} in a more general regime. See the supplementary material for other comparisons.

%
%
\subsection{Iteratively imputing the missing entries} \label{subsec:early_fusion_impute_iter}
When the number of missing nodes is important, filling missing entries with zero can lead to a huge bias and hence poor clustering performances. In order to reduce the bias we propose an alternative way of imputing the missing values (outlined as Algorithm \ref{alg:sum_of_adj_c}) based on the fact that each adjacency matrix is a noisy realization of a structured matrix. 

At iteration $t$, given an initial estimate $\hat{U}_K^{t} \in \mathbb{R}^{n\times K}$ of the common subspace we can estimate the membership matrix $\hat{Z}^t$ by applying $k$-means on  $\hat{U}_K^{t}$. Then, we can estimate the connectivity matrix $\hat{\Pi}^{(l),t}$ for each $l$ as
\begin{equation}\label{eq:sum_of_adj_c}
    \hat{\Pi}^{(l),t}=((\hat{Z}^t)^T\hat{Z}^t)^{-1}(\hat{Z}^t)^TA^{(l),t}\hat{Z}^t((\hat{Z}^t)^T\hat{Z}^t)^{-1}.
\end{equation}
Given $\hat{Z}^t$ and  $\hat{\Pi}^{(l),t}$ we estimate the rows and columns corresponding to missing nodes. Indeed, the connectivity profile of a node $i$ in layer $l$ is given by the $i$th row of $\hat{Z}^t\hat{\Pi}^{(l),t}(\hat{Z}^{t})^{T}$. By replacing the rows and columns of missing nodes by their estimated profiles, and leaving the value of observed nodes unchanged, we obtain the updated imputed matrix $A^{(l),t+1}$. Applying spectral clustering on $L^{-1}\sum_l A^{(l),t+1}$ then leads to an updated estimate $\hat{U}_K^{t+1}$ of the common subspace.  The procedure can be repeated using $\hat{U}_K^{t+1}$ and $A^{(l),t+1}$, thus iteratively imputing the missing values in order to obtain ``completed'' adjacency matrices that share the same $K$ rank structure across layers. \gb{In the worst case, the complexity of the algorithm run with $T$ iterations is $O((K+L)n^2 T+LKnT)$.} 
%

Similar iterative imputation methods have been studied in the context of principal component analysis, see for e.g., \cite{hetePCA, primePCA}.
\begin{algorithm}[!ht]
\caption{Sum of adjacency matrices with missing entries filled iteratively}\label{alg:sum_of_adj_c}
\begin{flushleft}
        \textbf{Input:} Number of communities $K$; $J_l$ and $A_{J_l}\in \mathbb{R}^{n\times n}$ for each $l$; initial estimate of the common subspace $\hat{U}_K^{0}\in \mathbb{R}^{n\times K}$ (with orthonormal columns)    obtained from Algorithm \ref{alg:sum_of_adj0}; number of iterations $T$.
\end{flushleft}
\begin{algorithmic}[1]
\State Initialize $t=0$ and $A^{(l), 0}=A_{J_l}$ for all $l$.
\Repeat 
\State Given $\hat{U}_K^{t}$, estimate the membership matrix $\hat{Z}^t$ and the connectivity parameters $\hat{\Pi}^{(l),t}$ for all $l$ by using \eqref{eq:sum_of_adj_c}.

\State For each $l$, replace rows (and corresponding columns) of $A^{(l)}$ corresponding to a missing node $i$ by the $i$th row of $\hat{Z}^t\hat{\Pi}^{(l),t}\hat{Z}^{t^T}$ to form  $A^{(l),t+1}$.

\State Compute the eigenvector matrix $\hat{U}_K^{t+1} = [u_1^{t+1} \ u_2^{t+1} \  \cdots \ u_K^{t+1}]$  associated with the $K$ largest (in absolute order) eigenvalues of $L^{-1}\sum_l A^{(l),t+1}$. Update $t\leftarrow t+1$.

\Until{$t\leq T$} 
\State Apply $K$-means on $\hat{U}_K^{T}$ to get a partition of $\calN$.
\end{algorithmic}
\textbf{Output:} A partition of the nodes $\calN =\cup_{i=1}^K \calC_i$.
\end{algorithm}
In our experiments, Algorithm \ref{alg:sum_of_adj_c} is seen to perform significantly better than other methods when $\rho$ decreases. While we do not currently have any statistical performance guarantee for Algorithm \ref{alg:sum_of_adj_c}, establishing this is an interesting direction for future work.
\section{Intermediate fusion methods: OLMF estimator} \label{sec:inter_fusion_olmf} 
Orthogonal linked matrix factorization (OLMF) is a clustering method for multilayer graphs that originated in the work of \cite{dhillon} in the complete data setup, and was later analysed in \cite{paul2020}. It shows good performance in various settings and outperforms spectral clustering when the multilayer network contains homophilic and heterophilic communities (see the numerical experiments in \cite{paul2020}). 
%

\subsection{The complete data setting} \label{subsec:olmf_complete}
In the complete data setting, the OLMF estimator is a solution of the following optimization problem 
\begin{equation}\label{eq:Olmf}
    (\hat{Q}, \hat{B}^{(1)}, \ldots , \hat{B}^{(L)}) \in \argmin{\substack{Q^TQ=I_k\\B^{(1)}, \ldots, B^{(L)}} }\sum_l|| A^{(l)}-QB^{(l)}Q^T||_F^2,
\end{equation} 
where $Q\in \mathbb{R}^{n\times K}$,  $B^{(l)}\in \mathbb{R}^{K\times K}$. Note that there is no constraint on the values taken by the entries of $B^{(l)}$.


A little algebra (see \cite{paul2020}) shows that the optimization problem \eqref{eq:Olmf} is equivalent to 
\begin{equation}\label{eq:Olmf2}
    \hat{Q}\in \argmax{Q^TQ=I_k}\sum_l||Q^TA^{(l)}Q||_F^2,\quad  \hat{B}^{(l)}=\hat{Q}^TA^{(l)}\hat{Q}
\end{equation}
for $l=1,\dots,L$. The OLMF estimator can be computed with a gradient descent on the Stiefel manifold (see \cite{paul2020} and supplementary material therein). The community estimation is then obtained by applying $K$-means on the rows of $\hat{Q}$. 

\subsection{Extension to the missing nodes setting} \label{subsec:olmf_missing}
We now present an extension of the OLMF estimator to the setting of missing nodes. By replacing the matrices $A^{(l)}$, $Q$ in the objective function in \eqref{eq:Olmf} with $A_{J_l} \in \mathbb{R}^{n \times n}$, $Q_{J_l} \in \mathbb{R}^{n \times K}$, we end up with the following modification for the incomplete setting

%
\begin{equation}\label{olmf:miss_def}
      (\hat{Q}, \hat{B}^{(1)}, \ldots , \hat{B}^{(L)})\in \argmin{\substack{Q^TQ=I_k\\B^{(1)}, \ldots, B^{(l)}}}\sum_l|| A_{J_l}-Q_{J_l}B^{(l)}Q_{J_l}^T||_F^2. 
\end{equation} 
In our experiments, we employ a BFGS algorithm for solving \eqref{olmf:miss_def}. \gb{The worst-case complexity of the algorithm is $O(LK(n^2+Kn))$.} Denoting the objective function in \eqref{eq:Olmf2} by $F$, its gradients are given by 
\begin{align*} 
\frac{\partial F}{\partial Q}&=-2\sum_l(A_{J_l}-Q_{J_l}B^{(l)}Q_{J_l}^T)Q_{J_l}B^{(l)}, \\
\frac{\partial F}{\partial B^{(l)}}&=-Q_{J_l}^T(A_{J_l}-Q_{J_l}B^{(l)}Q_{J_l}^T)Q_{J_l}.
\end{align*}
We relax the constraint that the gradient remains on the Stiefel manifold of $n \times k$ matrices, and  initialize the parameters using Algorithm \ref{alg:sum_of_adj0}.

The optimization problem in \eqref{eq:Olmf2} can be motivated via the missing nodes MLSBM as follows. If we replace the noisy realization $A_{J_l}$ with  $(Z\Pi^{(l)}Z^T)\odot \Omega^{(l)} $ then one can show (under some conditions) that the solution $\hat{Q}$
of \eqref{olmf:miss_def} has the same column span as the ground truth $Z \in \mathcal{M}_{n,K}$. This is shown formally in the following proposition.  
%
%
\begin{proposition} \label{prop:olmf}
Assume that $\Pi^{(l)}$ is full rank for each $l$, and that for each $l, l'$ the sets $J_l\cap J_{l'}$ intersect all communities. Then if $A_{J_l}=(Z\Pi^{(l)}Z^T)\odot \Omega^{(l)}$, it holds that the solution of \eqref{olmf:miss_def} is  given by $\hat{Q}=Z(Z^TZ)^{-1/2}$ and $\hat{B}^{(l)}=(Z^TZ)^{1/2}\Pi^{(l)}(Z^TZ)^{1/2}$ and is unique up to an orthogonal transformation. Moreover if $i$, $j$ belong to the same community, then $\hat{Q}_{i*}=\hat{Q}_{j*}$.
\end{proposition}

%
The matrix $\expec(A^{(l)})$ can be considered as a slight perturbation of $Z\Pi^{(l)}Z^T$ 
since the former has zeros on the diagonal. Thus the proposition shows that when there is no noise, the column-span of $\hat{Q}$ (the solution of \eqref{olmf:miss_def}) is the same as the ground truth partition $Z$.
\section{Numerical experiments} \label{sec:num_exps}
\subsection{Synthetic data} \label{subsec:num_syn_exps}
We now describe simulation results when the multilayer graph is generated from the missing nodes MLSBM. The normalized mutual information (NMI) criterion is used to compare the estimated community to the ground truth partition. It is an information theoretic measure of similarity taking values in $[0,1]$, with $1$ denoting a perfect match, and $0$ denoting completely independent partitions. Nodes that are not observed at least once are removed. 
The diagonal (resp. off-diagonal) entries of the connectivity matrices are generated uniformly at random over $[0.18,0.19]$ (resp. $0.7*[0.18,0.19]$). The ground truth partition is generated from a multinomial law with parameters $1/K$. While $K=3$ is fixed throughout, the parameters $n,\rho$ and $L$ are varied suitably. The average NMI is reported over $20$ Monte Carlo trials. As shorthand, we denote Alg. \ref{alg:k-pod} by \kpod, Alg. \ref{alg:sum_of_adj0} by \sumadjz, Alg. \ref{alg:sum_of_adj_c} by \sumadji, and \eqref{olmf:miss_def} by \olmfm.
 
 
Figure \ref{plot_rho} shows that \sumadjz~ gives good results unless $\rho$ is too small. Then, the performance of this method decreases quickly. This suggests that there is a threshold involving $\rho$ and the difference between intra and inter connectivity parameters. Figure \ref{plot_n} supports this claim. When $\rho$ is small, the performance of \sumadjz~doesn't improve when $n$ increases. So even if the separation between communities improves, the intra and inter connectivity parameters remain the same suggesting a link between these parameters and  $\rho$.
 
When $L$ increases (see Figs. \ref{plot_rho} and \ref{plot_L}), the performance of all methods improves. However, performance of \kpod~  improves less quickly than other methods. This is expected since contrary to other methods, \kpod~ relies more on the quality of each individual layer. 
\olmfm~ and \sumadji~ exhibit better performance than others in the challenging situation when $\rho$ is small, and perform as well as the others when $\rho \approx 1$. They  perform significantly better than \kpod, especially when $L$ is large.
  \begin{figure}[!ht]
  \center
   \includegraphics[scale=0.37]{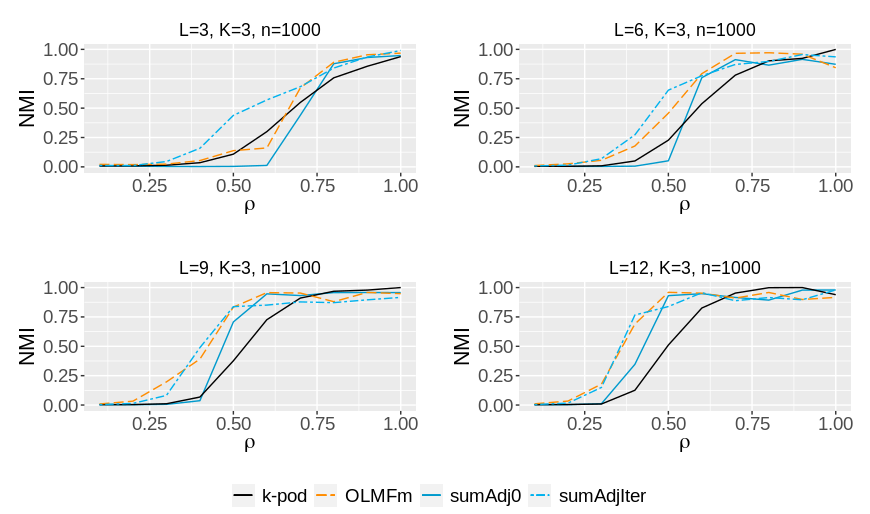}
   \caption{   NMI vs $\rho$ for different values of $L$}
   \label{plot_rho}
\end{figure}

 \begin{figure}[!ht]
 \center
   \includegraphics[scale=0.37]{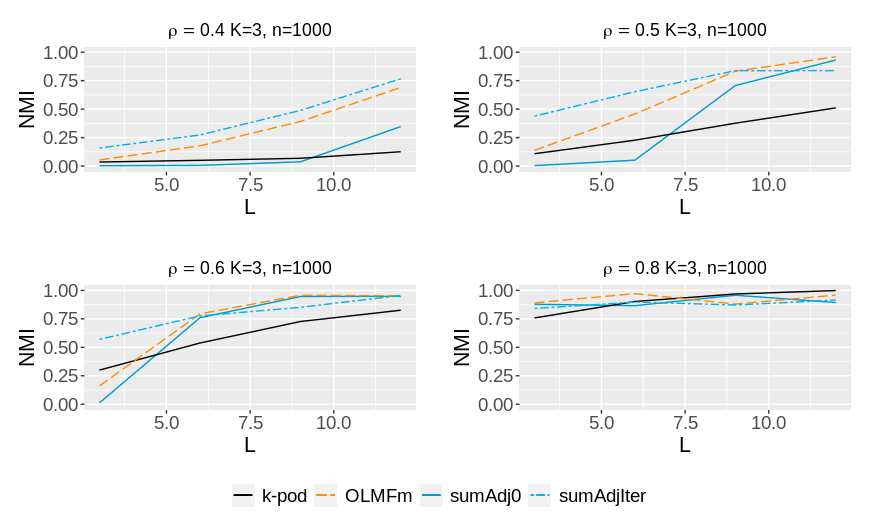}
     \caption{ NMI vs $L$ for different values of $\rho$}
     \label{plot_L} 
\end{figure}

 \begin{figure}[!ht]
 \center
   \includegraphics[scale=0.37]{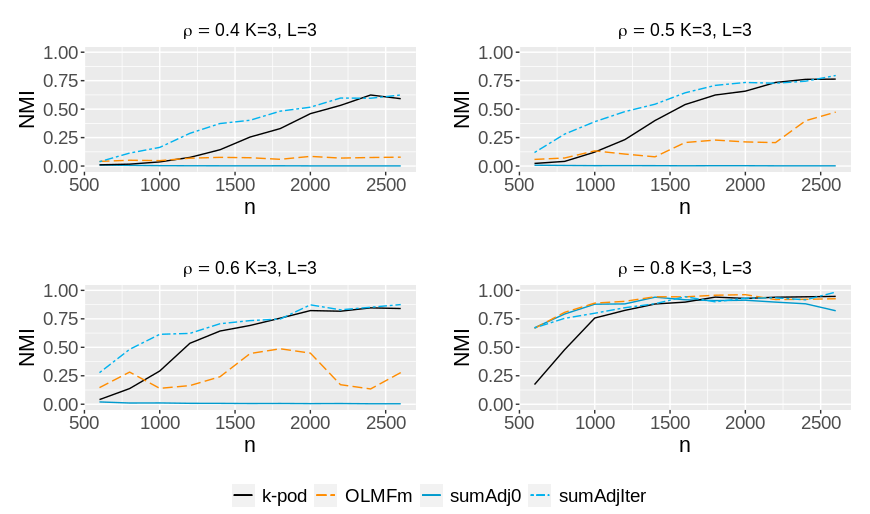}
     \caption{  NMI vs $n$ for different values of $\rho$}
     \label{plot_n}
\end{figure}

\subsection{MIT Reality Mining dataset} \label{subsec:mit_mining_exps}
This dataset records interactions (measured by cell phones activities) between $96$ students and staff at MIT in the 2004-05 school year (see \cite{mit}). We used the dataset as provided by the R package `GreedySTBM'.  As in \cite{airoldi2014} we removed the first and last layers, then discretized the time into one week intervals. The number of times two persons had an interaction during the week is not conserved in order to have a simple undirected graph corresponding to each layer. In total we obtained $32$ layers. 
For different values of $\rho$, we randomly removed nodes in each layer of the multilayer network. The average NMI over $50$ Monte Carlo trials is reported in Table \ref{table:mit} for our methods. The ground truth partition here is taken to be that obtained from \sumadjz~ when $\rho=1$. 
\begin{table}[!ht]
\centering
\begin{tabular}{rrrr}
  \hline
 $\rho$ & \sumadjz & \olmfm & \sumadji \\ 
  \hline
  1 & 1.00 & 1.00 & 1.00 \\ 
  0.9 & 0.99 & 0.96 & 0.99 \\ 
  0.8 & 0.97 & 0.86 & 0.97 \\ 
  0.7 & 0.96 & 0.93 & 0.96 \\ 
  0.6 & 0.94 & 0.79 & 0.94 \\
  0.5 & 0.89 & 0.91 & 0.90 \\
  0.4 & 0.76 & 0.73 & 0.78 \\
  0.3 & 0.56 & 0.57 & 0.62 \\
  0.2 & 0.26 & 0.41 & 0.36 \\
  0.1 & 0.09 & 0.10 & 0.11 \\ 
   \hline
\end{tabular}
\caption{NMI vs $\rho$ for MIT Reality Mining dataset}
\label{table:mit}
\end{table}
We disregarded \kpod~ because even when $\rho = 1$, its performance was disappointing and very sensitive to the initialization. This is not very surprising since this method works only if each layer is informative enough while we have a multilayer network where individual layers can be very sparse.

The performance of the other three methods studied are quite similar when $\rho$ is not too small ($\rho \geq 0.4$). However, the performance of \olmfm~ seems to be quite sensitive to initialization since for $\rho \in \set{0.6,0.8}$ its performance is worse than \sumadjz~ and \sumadji. Even if we remove half of the nodes in each layer we can still approximately recover the partition. 

\subsection{Malaria parasite genes network} \label{subsec:malaria_parasite_exps}
The dataset was constituted by \cite{malaria} to study the var genes parasite Plasmodium falciparum involved in Malaria. 
The nodes of the dataset correspond to $307$ different amino acid sequences and each of the $9$ layers corresponds to a highly variable region (HVR). Two nodes are linked in a given layer if there is a common block sequence between the corresponding amino acid sequences within the HVR associated to the layer. The analysis in \cite{malaria} and \cite{regtens} shows that the first six layers share the same community structure with $K=4$. Hence we restrict our study to the first six layers with $K=4$. We use the same procedure as before to delete nodes and to select the ground truth partition. \kpod~ was disregarded for the same reason as the previous experiment.
\begin{table}[!ht]
\centering 
\begin{tabular}{rrrrr}
  \hline
  $\rho$ & \sumadjz & \olmfm & \sumadji \\ 
  \hline
  1 & 1.00  & 0.99 & 1.00 \\ 
  0.9 & 0.75 & 0.75 & 0.72 \\ 
  0.8 & 0.63 & 0.62 & 0.58 \\ 
  0.7 & 0.47 & 0.49 & 0.47 \\ 
  0.6 & 0.32 & 0.37 & 0.34 \\ 
  0.5 & 0.22 & 0.20 & 0.26 \\ 
  0.4 & 0.13 & 0.07 & 0.16 \\
   \hline
\end{tabular}
\caption{NMI vs $\rho$ for Malaria parasite genes network}
\label{table:malaria}
\end{table}
As $\rho$ decreases, the clustering performance decreases rapidly due to a weak separation between the clusters as shown in Table \ref{table:malaria}. 

\section{Future work} \label{sec:rel_fut_work}
Our theorems require different conditions for consistency (each layer has to be informative enough for Algorithm \ref{alg:k-pod} and $L$ has to be large for Algorithm \ref{alg:sum_of_adj0}). It would be interesting to gain a better understanding of the fundamental limit of clustering with missing nodes. \gb{In this regard the use of two-round algorithms (see for e.g., \cite{AbbeSBM}) that do local refinement after having found a global partition could improve the misclustering rate.
It would also be interesting to consider model-based approaches by considering variational methods (\cite{daudin2008}) or Stochastic-EM algorithms (\cite{celeux1996stochastic}). } 

We assumed for simplicity that the nodes are missing under a Bernoulli sampling scheme, but other missing patterns could be considered. 
Another important direction would be to relax the strong condition imposed by MLSBM that all layers share the same common partition. For example, it would be more realistic to assume that the partition of networks evolving over time also evolves slowly.

\bibliography{references}

\begin{thebibliography}{33}
\providecommand{\natexlab}[1]{#1}
\providecommand{\url}[1]{\texttt{#1}}
\expandafter\ifx\csname urlstyle\endcsname\relax
  \providecommand{\doi}[1]{doi: #1}\else
  \providecommand{\doi}{doi: \begingroup \urlstyle{rm}\Url}\fi

\bibitem[Abbe(2018)]{AbbeSBM}
E.~Abbe.
\newblock Community detection and stochastic block models.
\newblock \emph{Foundations and Trends® in Communications and Information
  Theory}, 14\penalty0 (1-2):\penalty0 1--162, 2018.

\bibitem[Bandeira and van Handel(2016)]{bandeira2016}
A.~S. Bandeira and R.~van Handel.
\newblock Sharp nonasymptotic bounds on the norm of random matrices with
  independent entries.
\newblock \emph{The Annals of Probability}, 44\penalty0 (4):\penalty0
  2479--2506, 07 2016.

\bibitem[Bhattacharyya and Chatterjee(2018)]{Bhattacharyya2018SpectralCF}
S.~Bhattacharyya and S.~Chatterjee.
\newblock Spectral clustering for multiple sparse networks: I.
\newblock \emph{arXiv}, 1805.10594, 2018.

\bibitem[Bhattacharyya and Chatterjee(2020)]{bhattacharyya2020general}
S.~Bhattacharyya and S.~Chatterjee.
\newblock General community detection with optimal recovery conditions for
  multi-relational sparse networks with dependent layers.
\newblock \emph{arXiv}, 2004.03480, 2020.

\bibitem[Braun et~al.(2015)Braun, Schäfer, Walter, Erk, Romanczuk-Seiferth,
  Haddad, Schweiger, Grimm, Heinz, Tost, Meyer-Lindenberg, and Bassett]{brain}
U.~Braun, A.~Schäfer, H.~Walter, S.~Erk, N.~Romanczuk-Seiferth, L.~Haddad,
  J.~Schweiger, O.~Grimm, A.~Heinz, H.~Tost, A.~Meyer-Lindenberg, and
  D.~Bassett.
\newblock Dynamic reconfiguration of frontal brain networks during executive
  cognition in humans.
\newblock \emph{Proceedings of the National Academy of Sciences of the United
  States of America}, 112, 08 2015.

\bibitem[Buldygin and Moskvichova(2013)]{buldygin2013}
V.~Buldygin and K.~Moskvichova.
\newblock The sub-gaussian norm of a binary random variable.
\newblock \emph{Theory of Probability and Mathematical Statistics},
  86:\penalty0 33--49, 2013.

\bibitem[Celeux et~al.(1996)Celeux, Chauveau, and
  Diebolt]{celeux1996stochastic}
G.~Celeux, D.~Chauveau, and J.~Diebolt.
\newblock Stochastic versions of the em algorithm: an experimental study in the
  mixture case.
\newblock \emph{Journal of statistical computation and simulation}, 55\penalty0
  (4):\penalty0 287--314, 1996.

\bibitem[Chi et~al.(2015)Chi, Chi, and Baraniuk]{kpod}
J.~Chi, E.~Chi, and R.~Baraniuk.
\newblock k -pod a method for k -means clustering of missing data.
\newblock \emph{The American Statistician}, 70:\penalty0 1--29, 2015.

\bibitem[Daudin et~al.(2008)Daudin, Picard, and Robin]{daudin2008}
J.-J. Daudin, F.~Picard, and S.~Robin.
\newblock A mixture model for random graph.
\newblock \emph{Statistics and Computing}, 18:\penalty0 173--183, 06 2008.
\newblock \doi{10.1007/s11222-007-9046-7}.

\bibitem[Eagle and Pentland(2006)]{mit}
N.~Eagle and A.~Pentland.
\newblock Reality mining: Sensing complex social systems.
\newblock \emph{Personal Ubiquitous Comput.}, 10\penalty0 (4):\penalty0
  255–268, 2006.

\bibitem[Fortunato(2009)]{fortunato2010}
S.~Fortunato.
\newblock Community detection in graphs.
\newblock \emph{Physics Reports}, 486, 2009.

\bibitem[Giraud and Verzelen(2019)]{kmeans_relax}
C.~Giraud and N.~Verzelen.
\newblock Partial recovery bounds for clustering with the relaxed $k$-means.
\newblock \emph{Mathematical Statistics and Learning}, 1:\penalty0 317--374, 05
  2019.

\bibitem[Han et~al.(2015)Han, Xu, and Airoldi]{airoldi2014}
Q.~Han, K.~Xu, and E.~Airoldi.
\newblock Consistent estimation of dynamic and multi-layer block models.
\newblock In \emph{Proceedings of the 32nd International Conference on
  International Conference on Machine Learning - Volume 37}, page 1511–1520,
  2015.

\bibitem[Holland et~al.(1983)Holland, Laskey, and Leinhardt]{HOLLAND1983109}
P.~W. Holland, K.~B. Laskey, and S.~Leinhardt.
\newblock Stochastic blockmodels: First steps.
\newblock \emph{Social Networks}, 5\penalty0 (2):\penalty0 109 -- 137, 1983.

\bibitem[Hu and Chen(2019)]{hu2019onepass}
M.~Hu and S.~Chen.
\newblock One-pass incomplete multi-view clustering.
\newblock In \emph{The Thirty-Third Conference on Artificial Intelligence},
  pages 3838--3845, 2019.

\bibitem[Jing et~al.(2020)Jing, Li, Lyu, and Xia]{regtens}
B.-Y. Jing, T.~Li, Z.~Lyu, and D.~Xia.
\newblock Community detection on mixture multi-layer networks via regularized
  tensor decomposition.
\newblock \emph{arXiv}, 2002.04457, 2020.

\bibitem[Kim and Lee(2015)]{Kim2015}
J.~Kim and J.-G. Lee.
\newblock Community detection in multi-layer graphs: A survey.
\newblock \emph{SIGMOD Record}, 44:\penalty0 37--48, 2015.

\bibitem[Kivelä et~al.(2014)Kivelä, Arenas, Barthelemy, Gleeson, Moreno, and
  Porter]{kivela_eco}
M.~Kivelä, A.~Arenas, M.~Barthelemy, J.~P. Gleeson, Y.~Moreno, and M.~A.
  Porter.
\newblock {Multilayer networks}.
\newblock \emph{Journal of Complex Networks}, 2\penalty0 (3):\penalty0
  203--271, 2014.

\bibitem[Kumar et~al.(2011)Kumar, Rai, and Daume]{coregSC}
A.~Kumar, P.~Rai, and H.~Daume.
\newblock Co-regularized multi-view spectral clustering.
\newblock In J.~Shawe-Taylor, R.~S. Zemel, P.~L. Bartlett, F.~Pereira, and
  K.~Q. Weinberger, editors, \emph{Advances in Neural Information Processing
  Systems 24}, pages 1413--1421. Curran Associates, Inc., 2011.

\bibitem[Larremore et~al.(2013)Larremore, Clauset, and Buckee]{malaria}
D.~Larremore, A.~Clauset, and C.~Buckee.
\newblock A network approach to analyzing highly recombinant malaria parasite
  genes.
\newblock \emph{PLoS computational biology}, 9:\penalty0 e1003268, 2013.

\bibitem[Lei(2020)]{lei2020tail}
J.~Lei.
\newblock Tail bounds for matrix quadratic forms and bias adjusted spectral
  clustering in multi-layer stochastic block models.
\newblock \emph{arXiv}, 2003.08222, 2020.

\bibitem[Lei and Rinaldo(2015)]{lei2015}
J.~Lei and A.~Rinaldo.
\newblock Consistency of spectral clustering in stochastic block models.
\newblock \emph{The Annals of Statistics}, 43\penalty0 (1):\penalty0 215--237,
  02 2015.

\bibitem[Lei et~al.(2019)Lei, Chen, and Lynch]{jinLei2019}
J.~Lei, K.~Chen, and B.~Lynch.
\newblock {Consistent community detection in multi-layer network data}.
\newblock \emph{Biometrika}, 107\penalty0 (1):\penalty0 61--73, 2019.

\bibitem[{Liu} et~al.(2020){Liu}, {Li}, {Tang}, {Xia}, {Xiong}, {Liu}, {Kloft},
  and {Zhu}]{imvc}
X.~{Liu}, M.~{Li}, C.~{Tang}, J.~{Xia}, J.~{Xiong}, L.~{Liu}, M.~{Kloft}, and
  E.~{Zhu}.
\newblock Efficient and effective regularized incomplete multi-view clustering.
\newblock \emph{IEEE Transactions on Pattern Analysis and Machine
  Intelligence}, pages 1--1, 2020.

\bibitem[Mitzenmacher and Upfal(2005)]{random_algBook}
M.~Mitzenmacher and E.~Upfal.
\newblock \emph{Probability and Computing: Randomized Algorithms and
  Probabilistic Analysis}.
\newblock Cambridge University Press, USA, 2005.
\newblock ISBN 0521835402.

\bibitem[Paul and Chen(2020)]{paul2020}
S.~Paul and Y.~Chen.
\newblock Spectral and matrix factorization methods for consistent community
  detection in multi-layer networks.
\newblock \emph{The Annals of Statistics}, 48\penalty0 (1):\penalty0 230--250,
  2020.

\bibitem[Pensky and Zhang(2019)]{pensky2019}
M.~Pensky and T.~Zhang.
\newblock Spectral clustering in the dynamic stochastic block model.
\newblock \emph{Electron. J. Statist.}, 13\penalty0 (1):\penalty0 678--709,
  2019.

\bibitem[Tang et~al.(2009)Tang, Lu, and Dhillon]{dhillon}
W.~Tang, Z.~Lu, and I.~Dhillon.
\newblock Clustering with multiple graphs.
\newblock In \emph{IEEE International Conference on Data Mining}, pages
  1016--1021, 2009.

\bibitem[Tropp(2012)]{user_friendly}
J.~A. Tropp.
\newblock User-friendly tail bounds for sums of random matrices.
\newblock \emph{Found. Comput. Math.}, 12\penalty0 (4):\penalty0 389–434,
  Aug. 2012.
\newblock ISSN 1615-3375.

\bibitem[Vershynin(2016)]{vershynin}
R.~Vershynin.
\newblock Four lectures on probabilistic methods for data science.
\newblock 12 2016.

\bibitem[Wedin(1972)]{wedin1972perturbation}
P.-{\AA}. Wedin.
\newblock Perturbation bounds in connection with singular value decomposition.
\newblock \emph{BIT Numerical Mathematics}, 12\penalty0 (1):\penalty0 99--111,
  1972.

\bibitem[Zhang et~al.(2018)Zhang, Cai, and Wu]{hetePCA}
A.~Zhang, T.~T. Cai, and Y.~Wu.
\newblock Heteroskedastic pca: Algorithm, optimality, and applications.
\newblock \emph{arXiv}, 1810.08316, 2018.

\bibitem[Zhu et~al.(2019)Zhu, Wang, and Samworth]{primePCA}
Z.~Zhu, T.~Wang, and R.~J. Samworth.
\newblock High-dimensional principal component analysis with heterogeneous
  missingness.
\newblock \emph{arXiv}, 1906.12125, 2019.

\end{thebibliography}

\clearpage
\appendix
\onecolumn
\begin{center}
\Large \textbf{Supplementary Material}
\end{center}
The proof of Theorem \ref{thm:kpod} is presented in Appendix \ref{app:thm1} and that of Theorem \ref{thm:sumadj} is presented in Appendix \ref{app:thm2}. Proposition \ref{prop:olmf} is proved in Appendix \ref{app:olmf} and auxiliary lemmas are gathered in Appendix \ref{app:aux_lemmas}. Appendix \ref{app:miss_edges} is devoted to discussing the missing edges setting. Existing bounds for the misclustering rate under the MLSBM in the complete setting are gathered in Appendix \ref{app:miscl_bounds}. 


\section{Proof of Theorem \ref{thm:kpod}}
\label{app:thm1}
Let $\hat{Z}$ and $\hat{C}$ be solutions of the optimization problem \eqref{final:eq1} and write $\Bar{U}:=\hat{Z}\hat{C}$. Define $U'$ as the block matrix obtained by stacking the matrices $U^{(l)}O_l$, $L_i=\lbrace l \in [L] : i \in J_l \rbrace$ be the indices of layers where the node $i$ appears, and $\calN_u=\lbrace i : |L_i|\geq \rho L/c \rbrace$ where $c>1$ is a constant that will be fixed later.
Let $\calS_k$ be the set of `bad nodes' defined as  \[ \calS_k := \lbrace i\in \calC_k\cap \calN_u : \forall l\in L_i,\, ||U_{i*}^{(l)}O_l-\bar{U}^{(l)}_{i*} ||\geq \delta_k^{(l)}/2 \rbrace \] where \[ \delta_k^{(l)}:=\min_{\substack{i\in \calC_k\\ i'\in \calC_{k'} \\ k'\neq k}}||U_{i'*}^{(l)}-U_{i*}^{(l)} ||=\min_{\substack{i\in \calC_k\\ i'\in \calC_{k'} \\ k'\neq k}}||U_{i'*}^{(l)}O_l-U_{i*}^{(l)}O_l || \]is the smallest distance between two rows of $U^{(l)}$ corresponding to different communities. Let $\calT_k :=(\calC_k\setminus \calS_k)\cap \calN_u$ be the complement of $\calS_k$ in $\calN_u \cap \calC_k$ and $\calT = \cup_k \calT_k$.

\textbf{Step 1.} First let us show by contradiction that if for all $k$, $|\calT_k|> n_k/30$ and $n_k$ satisfies the  assumptions of the theorem, then all the nodes in $\calT$ are well classified with probability at least $1-O(n^{-1})$. Assume that there exist $i\in \calT_k$ and $j\in \calT_{k'}$ such that $\bar{U}_i=\bar{U}_j$. If $L_i\cap L_j\neq \varnothing$, every $l\in L_i\cap L_j$ satisfies \[ \max(\delta_k^{(l)}, \delta_{k'}^{(l)})\leq || U_{i*}^{(l)}-U_{j*}^{(l)}||\leq || U_{i*}^{(l)}-\bar{U}_{i*}^{(l)} || +||U_{j*}^{(l)}-\bar{U}_{j*}^{(l)} ||<\frac{\delta_k^{(l)}}{2}+\frac{\delta_{k'}^{(l)}}{2}\] contradicting the fact that $i\in \calT_k$ and $j\in \calT_{k'}$. It remains to treat the case $L_i\cap L_j= \emptyset$. Let $C_1$ be a cluster induced by $\bar{U}$ containing the nodes $i$ and $j$. If there were other nodes belonging to $\calC_k$ and $\calC_{k'}$ but appearing in a common layer, the previous argument can be used to obtain a contradiction. So we can assume that all the nodes of community $\calC_k$ in $C_1$ and all nodes of community $\calC_{k'}$ in $C_1$ appear on distinct layers. We are going to show this property implies that for all $k$ the size of $\calC_k\cap C_1$, and thus the size of $C_1$, is small with high probability. Let $l_1$ be a layer where a node in $\calC_{k'}\cap C_1$ appears. The probability that none of the nodes in $\calC_k\cap C_1$ appear in $l_1$ is $(1-\rho)^{|\calC_k\cap C_1 |}$ and this probability is $O(1/n^2)$ if $|\calC_k\cap C_1 |\geq 2\rho^{-1} \log n$ (we used the fact that $-\log (1-\rho)\geq \rho$). By symmetry, the result holds for every $k$ such that $|\calC_k\cap C_1 |>0$. Therefore we can assume that $|C_1\cap \calC_k|\leq 2\rho^{-1} \log n$. 
Since for all $k$, $|\calT_k|\geq n_k/30 \geq 3K^2\rho^{-1} \log n$ by assumption, there are nodes in $\calT_k$ and $\calT_{k'}$ that are not in $C_1$. Hence there is another cluster $C_2$ induced by $\hat{U}$ containing nodes from two different communities. The same argument can be applied to $C_2$ and iteratively to $C_3, \ldots, C_K$. 
At the end, since the $C_k$ form a partition of the set of nodes, we obtain  \[ |\calT_{k'}|=\sum_k|C_k\cap \calT_{k'}|\leq 2K^2\rho ^{-1}\log n\] contradicting the fact that $|\calT_k|\geq 3K^2\rho^{-1} \log n$.

We are now going to show that under the  assumptions of the theorem, for all $k$, $\calT_k$ satisfies $|\calT_k|> n_k/30$ with probability at least $1-O(n^{-1})$. In order to prove this result we will first show that $|\calS_k|$ is small (Step 2) and then show that $\calN_u \cap \calC_k$ is large (Step 3).

\textbf{Step 2.} Observe that if $i\in \calS_k$ then $\forall l\in L_i$, $4(\delta_k^{(l)})^{-2}||(U^{(l)}O_l)_{i*}-\bar{U}^{(l)}_{i*} ||^2\geq 1$. So for all $k$, %
\begin{equation}
\label{th1:maj}
    |\calS_k|\delta_k^2 \leq 4\sum_{i \in \calC_k\cap \calN_u} \min_{l \in L_i}||(U^{(l)}O_l)_{i*}-\bar{U}^{(l)}_{i*} ||^2 \leq 4\sum_{i \in \calC_k\cap \calN_u} \frac{\sum_{l\in L_i}||(U^{(l)}O_l)_{i*}-\bar{U}^{(l)}_{i*} ||^2 }{|L_i|}
\end{equation}  
where we used the fact $\delta_k^{(l)}\geq \delta_k$ for the first inequality, and the fact that the minimum is always bounded by the mean for the second inequality.

By summing over $k$, and using the fact that $|L_i|\geq \rho L/c$ for $i\in \calN_u$,  we get \begin{equation}\label{eq:temp2}
    \sum_k|\calS_k|\delta_k^2 \leq \frac{4c}{\rho L}\sum_{i\in \calN_u}\sum_{l\in L_i}||(U^{(l)}O_l)_{i*}-\bar{U}^{(l)}_{i*} ||^2 \leq   \frac{C}{\rho L} ||(U'-\bar{U})\odot \Omega_U ||_F^2.
\end{equation} 

 Using triangular inequality we get \begin{equation}\label{eq:temp3}
     || (U'-\bar{U})\odot \maskkpod ||_F^2 \leq ||(U'-\hat{U})\odot \maskkpod||_F^2+|| (\hat{U}-\bar{U})\odot \maskkpod||_F^2 \leq 2 || (\hat{U}-U')\odot \maskkpod||_F^2
 \end{equation} where the second inequality follows from the fact that $U'$ is feasible for \eqref{final:eq1}, i.e., it can be written as a product of a membership matrix $Z$ and a centroid matrix $C\in \mathbb{R}^{K\times KL}$.
 
Notice that \[ ||(\hat{U}-U')\odot \maskkpod||_F^2 =\sum_l ||\hat{U}_{J_l}-U_{J_l}O_l ||_F^2.\] Let $\lambda_{K,J_l}$ be the $K$th largest singular value of $Z_{J_l}\Pi^{(l)}Z_{J_l}^T$. This last quantity depends on the missing patterns, but the concentration results established in Lemma \ref{lem:sizenk} shows that for all $l$, $n_{J_l}\leq 1.5\rho n$ with probability at least $1-O(n^{-1})$ and Lemma \ref{lem:eigval_min} applied with $Z_{J_l}$ instead of $Z$ and $n_{J_l,min}$ instead of $n_{min}$ shows that $\lambda_{K,J_l}\geq n_{J_l,min} \lambda_K^{(l)}\geq 0.5\rho n_{min}\lambda_K^{(l)}$ with probability at least $1-O(n^{-1})$. The concentration inequality used in Lemma \ref{lem:concAdj} and Lemma \ref{lem:sizenk} show that with probability at least $1-O(n^{-1})$, $||A_{J_l}-\expec(A_{J_l}) ||\leq C \sqrt{n_{J_l}p_{max}^{(l)}} \leq C \sqrt{\rho np_{max}^{(l)}} $. But  $\rho n_{min}\lambda_K^{(l)} \geq 4 C\sqrt {\rho n p_{max}^{(l)}}$ for all $l$ due to our assumptions. Moreover, since with high probability, $n_{J_l} p_{max}^{(l)} \geq c \log n$ for each $l$ (using the fact that w.h.p, $n_{J_l} \geq c' \rho n$ for each $l$, the condition in the theorem statement suffices), hence Lemma \ref{lem:concAdj} applies and we get that for  for each $l$ that with probability $1-O(n^{-2})$ \begin{equation}\label{eq:temp5}
    ||\hat{U}_{J_l}-U_{J_l}O_l ||_F^2 \leq \frac{C||A_{J_l}-\expec(A_{J_l}) ||_F^2}{\lambda_{K,J_l}^2}\leq CK\frac{n_{J_l}p_{max}^{(l)}}{\lambda_{K,J_l}^2}.
\end{equation} 
 So by Lemma \ref{lem:eigval_min} and Lemma \ref{lem:sizenk} there exists $C > 0$ such that with probability at least $1-O(Ln^{-2})$ (via union bound), we have for all $l\leq L$ that \begin{equation}\label{eq:temp4}
    \frac{n_{J_l}p_{max}^{(l)}}{\lambda_{K,J_l}^2} \leq  C\frac{np_{max}^{(l)}}{\rho (n_{min}\lambda_K^{(l)})^2}.
\end{equation}
    
Plugging equations \eqref{eq:temp2}, \eqref{eq:temp3},  \eqref{eq:temp4} and \eqref{eq:temp5} into \eqref{th1:maj} we obtain with probability at least $1-O(n^{-1})$ \[ \sum_k |\calS_k|\delta_k^2 \leq CK \sum_l \frac{np_{max}^{(l)}}{\rho^2 L(n_{min}\lambda_K^{(l)})^2}. \] 
We have $\delta_k = \min_l \delta_k^{(l)}=\min_l \sqrt{\frac{1}{n_{k,J_l}}}$ by Lemma 2.1 in \cite{lei2015}. Moreover $\min_l \sqrt{\frac{1}{n_{k,J_l}}}\geq \frac{c}{\sqrt{\rho n_k}}$ with probability at least $1-O(n^{-1})$ by Lemma \ref{lem:sizenk}  since $\rho n_k\geq C\log^2n$ by assumption. Thus we obtain \[ \sum_k |\calS_k| \leq \sum_k |\calS_k|(c^{-1}\sqrt{\rho n_k})^2(\delta_k)^2\leq CKn_{max}\sum_l \frac{np_{max}^{(l)}}{\rho L(n_{min}\lambda_K^{(l)})^2}. \]
Observe that $\frac{n_{max}}{n}\leq \frac{\beta}{K}$.
If  \[ \sum_l \frac{np_{max}^{(l)}}{\rho L(n_{min}\lambda_K^{(l)})^2} < (30C \beta^2K)^{-1},\] then  $|\calS_K|<n_k/30 $ for all $k$
By using $n_{min}\geq \frac{n}{\beta K}$ this last condition can be simplified as \[  \frac{1}{\rho Ln}\sum_l \frac{p_{max}^{(l)}}{(\lambda_K^{(l)})^2} < (30C \beta^4K^3)^{-1}.\]

\textbf{Step 3.} We are now going to show that $|\calN_u\cap \calC_k|$ is large. Let $p(\rho, L)= \prob (|L_i|<\rho L/c)$. For the choice $c=25$, we always have $p<8/10$ since $\rho L \geq 1$ by assumption. Chernoff bound  (Lemma \ref{lem:chernoff}) shows that $p(\rho, L)\leq e^{-\rho L(1-c^{-1})/3}$. If $\rho L >12\log n$ then with probability at least $1-O(n^{-2})$, $\calN_u=\calN$ and $|\calN_u^c|=0$.
Let us assume that $\rho L <12\log n$. The number of nodes in $\calN_u^c\cap \calC_k$ can be written as a sum $n_k$ independent Bernoulli variables with parameter $p=p(\rho, L)$ (we will omit the dependence on $\rho$ and $L$ in the following for notation convenience): \[ |\calN_u^c\cap \calC_k|= \sum_{i\leq n_k} b_i.\] In expectation $\expec(|\calN_u^c\cap \calC_k|)=pn_k$ and Hoeffding's bound implies that $\prob(||\calN_u^c\cap \calC_k|-pn_k|\geq t)\leq 2e^{-t^2/n_k}$ for any choice of $t>0$.  So we can take $t=C\sqrt{n_k\log n}=o(n_k)$ and obtain that with probability at least $1-O(Kn^{-2})$ for all $k$ \[ |\calN_u^c\cap \calC_k| \leq n_kp+ C\sqrt{n_k\log n}.\] 
Thus $|\calC_k\cap \calN_u| \geq n_k(1-p-\sqrt{\frac{C\log n}{n_k}})$. If $n$ is large enough, then $\sqrt{\frac{C\log n}{n_k}}<1/30$.

Since the sets $\calS_k$ have cardinalities at most $\frac{n_k}{30}$ we obtain that $|\calT_k|\geq \frac{5n_k}{30}$. 

\textbf{Conclusion.}
Steps 1,2 and 3 show that all nodes that belong to $\calT_k$ are well classified with probability at least $1-O(n^{-1})$. Hence the number of misclustered nodes is bounded by the sum of the cardinalities of $\calS_k$ plus  $\abs{\calN_u^c}$. So with probability at least $1-O(n^{-1})$ we get 
\begin{align*}
     r(\hat{Z},Z)&\leq \frac{1}{n}(|\calN_u^c|+\sum_k |\calS_k|)\\
     &\leq \frac{31}{30}p(\rho, L)+ C \beta \sum_l\frac{np_{max}^{(l)}}{\rho L(n_{min}\lambda_K^{(l)})^2}\\
     &\leq C\exp(-c'\rho L)+\frac{C\beta^3K^2}{\rho Ln}\sum_l\frac{p_{max}^{(l)}}{(\lambda_K^{(l)})^2}.
\end{align*}

\section{Proof of Theorem \ref{thm:sumadj}}
\label{app:thm2}
 In order to prove Theorem \ref{thm:sumadj}, we are going to show that $\Tilde{A}$ is close to $\expec(\Tilde{A}|\Omega)$ with high probability for every realization of $\Omega$ and that  $\expec(\Tilde{A}|\Omega)$ concentrates around $\expec(\Tilde{A})$ if $L$ is large enough. These results are summarized in the following proposition.
 \begin{proposition}
 \label{prop:conc}There exist constants $c_1$ and $c_2$ such that the following holds.
 \begin{enumerate}
     \item $\prob (||\Tilde{A}-\expec(\Tilde{A}|\Omega)||\geq c_1  \rho^{-2}\left(\sqrt{\frac{np_{max}}{L}}+\sqrt{\frac{\log n}{L}})|\Omega \right)\leq n^{-1}$;
     \item $|| \expec(\Tilde{A}|\Omega) -\expec(\Tilde{A}) || \leq c_2(\rho^{-2}-1) \left[np_{max}\left(\sqrt{\frac{\log n}{L}}+\frac{\log n}{L}\right) \right]$ with probability at least $1-o(n^{-1})$.
 \end{enumerate}
\end{proposition}
\begin{proof}

The proof of the first statement is the same as the proof of the corresponding inequality if there are no missing values. Since we reason conditionally to the missingness mechanism, the zero entries of $\Tilde{A}$ can also be considered as the realization of independent Bernoulli variables with parameter zero.

Let $E=\rho^2(\Tilde{A}-\expec(\Tilde{A}|\Omega))$ and $E'$ be an independent copy of $E$. Define $E^s=E-E'$ as the symmetrized version of $E$. Jensen's inequality implies that $||E||=||\expec(E-E'|E) ||\leq \expec(||E^s|| \ | \ E)$, so it is enough to control $||E^s||$.

The $\psi_2$ norm (see for example \cite{vershynin}, Proposition 1.2.1) of each entry of $E^s$ is bounded by $K_L:=C\sqrt{L^{-1}}\calK$ where $\calK=\max_{i,j,l}||A_{ij}^{(l)}||_{\psi_2}$ and $A_{ij}^{(l)}$ are centered Bernoulli random variables with parameters $p_{ij}^{(l)}$.  By definition of the $\psi_2$ norm there exists a constant $c_0$ such that for each $i,j\leq n$ \[ \prob(|E^s_{ij}|\geq c_0K_L\sqrt{\log n})\leq n^{-4}.\]
 Define $T_{ij}=E^s_{ij}\indic_{|E^s_{ij}|\leq c_0K_L\sqrt{\log n}}$ and let $T=(T_{ij})\in \matR^{n\times n}$. By a  union bound argument the matrix $E^s-T$ has entries that are not zero with probability at most $n^{-2}$, thus $||E^s||=||T||$ with probability at least $1-O(n^{-2})$. Since the entries of $E^s$ are symmetric, the matrix $T$ is centered and has entries bounded by  $c_0K_L\sqrt{\log n}$ by construction. So we can apply the bound from Lemma \ref{bandeira} to $T$ and obtain \[ ||T|| \leq C \sqrt{\frac{np_{max}}{L}}+ K_L\log n\] with probability at least $1-O(n^{-1})$. 
 We can use the following theorem to get a sharp bound for $K_L$.
 
 \begin{theorem}[{\cite[Theorem 2.1, Lemma 2.1 (K6)]{buldygin2013}}]
 Let $Y$ be a centered Bernoulli random variable with parameter $p$, i.e., $Y = 1-p$ with probability $p$, and $Y = -p$ with probability $1-p$. Then, 
 \begin{equation*}
     \norm{Y}_{\psi_2}^2 = 
     \left\{
\begin{array}{rl}
0 \ ; & p \in \set{0,1}, \\
1/4 \ ; & p = 1/2, \\
\frac{1-2p}{2 \log(\frac{1-p}{p})} \ ; & p \in (0,1) \setminus \set{\frac{1}{2}}.
\end{array} \right.
 \end{equation*}
 In particular, it holds that $\norm{Y}_{\psi_2} \leq \frac{1}{\sqrt{2 \abs{\log (\min\set{2p,2(1-p)})}}}$.
 \end{theorem}
 
  If $np_{max}\leq \log^2 n$, then $K_L\leq C(L\log n)^{-1/2}$ and we obtain the first part of the proposition by dividing by $\rho^2$. If $np_{max}\geq \log^2 n$ then we can bound use the trivial bound $K_L\leq C L^{-1/2}$ to see that $K_L\log n\leq C\sqrt{\frac{np_{max}}{L}}$. 
  Hence 
  \[ ||\Tilde{A}-\expec(\Tilde{A}|\Omega)||\leq C  \rho^{-2} \left(\sqrt{\frac{np_{max}}{L}}+\sqrt{\frac{\log n}{L}} \right)\] with probability at least $1-O(1/n)$ for all $\Omega$.

It remains to bound the difference between $\expec(\Tilde{A}|\Omega)$ and $\expec(\Tilde{A})$. We do so using the matrix Bernstein inequality (Lemma \ref{mat_bern}). Let $X_l:=\rho^{-2}\expec (A^{(l)})\odot \Omega^{(l)}-\expec(A^{(l)})$; clearly each $X_l$ is centered. Moreover $||X_l||\leq ||X_l ||_F\leq p_{max}n (\rho^{-2}-1)$.

For notation convenience, we will write $X$ instead of $X_l$. We have $\expec(X^2)_{ij}=\sum_{k\leq n}X_{ik}X_{jk}$ because $X$ is symmetric. Recall that $X_{ik}=a_{ik}(\rho^{-2}\omega_i\omega_k-1)$ where $a_{ik}$ corresponds to $A^{(l)}_{ik}$. A simple calculation shows that \begin{align*}
    \expec(X^2)_{ij}&=\sum_k \expec(a_{ik}a_{jk}(\rho^{-2}\omega_i\omega_k-1)(\rho^{-2}\omega_j\omega_k-1)))\\
                    &=\sum_k a_{ik}a_{jk}\expec((\rho^{-2}\omega_i\omega_k-1)(\rho^{-2}\omega_j\omega_k-1))).
\end{align*}    
If $i=j$, $\expec((\rho^{-2}\omega_i\omega_k-1)^2)=\rho^{-2}-1$ and if $i\neq j$, $\expec((\rho^{-2}\omega_i\omega_k-1)(\rho^{-2}\omega_j\omega_k-1)))= \rho^{-1}-1$. So in both cases, $|\expec(X^2)_{ij}|\leq np_{max}^2(\rho^{-2}-1)$. We can now bound $||\expec(X^2_l) ||$ by $||\expec(X_l^2)||_F\leq [np_{max}(\rho^{-2}-1))]^2$ and $\sigma^2:=||\sum_l \expec(X_l^2)||$ by $L[np_{max}(\rho^{-2}-1)]^2$.

 Therefore matrix Bernstein inequality implies that \[ ||\sum_l X_l||\leq C(\rho^{-2}-1) (np_{max}\sqrt{L\log n}+np_{max}\log n)\] with probability at least $1-O(n^{-1})$ for a constant $C$ chosen appropriately.
 \end{proof}
 
 \begin{proof}[Proof of Theorem \ref{thm:sumadj}]
 Triangle inequality gives  $||\tilde{A}-\expec(\tilde{A})||\leq ||\tilde{A}-\expec(\tilde{A}|\Omega) ||+||\expec(\tilde{A}|\Omega)-\expec(\tilde{A})||$ and we can use  Proposition \ref{prop:conc} to bound with high probability each term. So with probability at least $1-O(n^{-1})$ \[ ||\tilde{A}-\expec(\tilde{A})|| \leq \frac{C}{\rho ^2}\left(\sqrt{\frac{np_{max}}{L}}+\sqrt{\frac{\log n }{L}}\right)+C(\rho^{-2}-1) \left(np_{max}\sqrt{\frac{\log n}{L}} +
\frac{np_{max} \log n}{L} \right).\]  We can now use the relation established in \cite[Lemma 2.1]{lei2015},  and a immediate adaptation of Lemma  \ref{lem:concAdj} to conclude as in Theorem \ref{thm:kpod}. 
 \end{proof}

 \section{Missing edges}
 \label{app:miss_edges}
Assume that each edge is observed independently with probability $\rho$. So we can write $\Omega^{(l)}=(w_{ij}^{(l)})_{i,j}$ where $w_{ij}^{(l)}\overset{\text{ind.}}{\sim} \mathcal{B}(\rho)$ for $i<j$. Let us denote $\tilde{A}=(L\rho)^{-1}\sum_l A^{(l)}\odot \Omega^{(l)}$, we then have $\expec(\tilde{A}|\Omega)= L^{-1}\sum_l \expec(A^{(l)})$.

We are going to show that in this setting $ \expec(\Tilde{A}|\Omega)$ concentrates around $\expec(\Tilde{A})$.  Contrary to the missing nodes setting the entries of  $ \expec(\Tilde{A}|\Omega)$ are independent.  Hence this matrix concentrates around its expectation faster than in the case where nodes are missing. This is shown in the following proposition.

\begin{proposition}
There exists a constant $C>0$ such that with probability $1-o(1)$, \[ || \expec (\Tilde{A}|\Omega) -\expec(\Tilde{A}) || \leq  C \left(\frac{p_{max}\sqrt{n}}{\rho \sqrt{ L}}+p_{max}\sqrt{\frac{\log n}{\rho^2 }} \right). \]
\end{proposition}


\begin{proof}
 We have $\expec(\tilde{A}|\Omega)_{ij}-\expec(\Tilde{A})_{ij}= (\rho L)^{-1}\sum_l a_{ij}^{(l)}(w_{ij}^{(l)}-\rho)$ where $a_{ij}^{(l)}=\expec(A^{(l)})_{ij}$. Hence the matrix $\expec(\tilde{A}|\Omega)-\expec(\Tilde{A})$ is centered and has independent subgaussian entries. 
 As in the proof of Proposition \ref{prop:conc}, we can use Remark 3.13 in \cite{bandeira2016}. Observe that \[\max_i \sqrt{\sum_j \Var(\expec(\tilde{A}|\Omega)_{ij}-\expec(\tilde{A})_{ij})} \leq \sqrt{\frac{n}{\rho^2L}}p_{max}\] and $||\expec(\tilde{A}|\Omega)_{ij}-\expec(\tilde{A})_{ij} ||_{\infty}\leq \rho^{-1}p_{max}$ .Therefore \[ || \expec (\Tilde{A}|\Omega) -\expec(\Tilde{A}) ||\leq C \rho^{-1} \left(\sqrt{\frac{np_{max}^2}{L}}+p_{max}\sqrt{\log n} \right).\] 
\end{proof}

The difference $||\tilde{A}-\expec(\tilde{A}|\Omega)||$ can be bound as in Proposition \ref{prop:conc}. With probability at least $1-O(n^{-1})$ \[ ||\tilde{A}-\expec(\tilde{A}|\Omega)|| \leq C\rho ^{-1} \left(\sqrt{\frac{np_{max}}{L}}+\sqrt{\frac{\log n}{L}} \right).\] Then we can conclude as in Theorem \ref{thm:sumadj} by using Lemma \ref{lem:concAdj} that if  $\lambda_K(\expec(\tilde{A}))=cnp_{max}$ and $np_{max}\geq c'\log n$ then with probability at least $1-O(n^{-1})$ \[ r(\hat{Z},Z)\leq C\frac{||\tilde{A}-\expec(\tilde{A}|\Omega)||+|| \expec (\Tilde{A}|\Omega) -\expec(\Tilde{A}) ||}{\lambda_K(\expec(\tilde{A}))}\leq  \frac{C}{\rho \sqrt{Lnp_{max}}}\] because the error due to missing values is negligible compared to the error due to the noise when $L\leq p_{max}^{-1}$, contrary to the missing nodes setting.
\section{Proof of Proposition \ref{prop:olmf}}
\label{app:olmf}
Observe that $(Z\Pi^{(l)}Z^T)\odot \Omega =Z_{J_l}\Pi^{(l)}Z_{J_l}^T$ and $(Z(Z^TZ)^{-1/2})_{J_l}=Z_{J_l}(Z^TZ)^{-1/2}$. Hence $\hat{Q}:=Z(Z^TZ)^{-1/2}$ and $\hat{B}^{(l)}:=(Z^TZ)^{1/2}\Pi^{(l)}(Z^TZ)^{1/2}$ are solutions of the optimization problem \eqref{olmf:miss_def}. Any other solution $(\hat{Q}', \hat{B}^{'(1)}, \ldots, \hat{B}^{'(L)})$ should cancel the objective function and satisfy $(\hat{Q}')^\top \hat{Q}'=I_K$ and for all $l\leq L$, \begin{equation}\label{prop1:eq}
    Z_{J_l}\Pi^{(l)}Z_{J_l}^T=\hat{Q}_{J_l}'\hat{B}^{'(l)}(\hat{Q}_{J_l}')^T .
\end{equation} 
Since $\Pi^{(l)}$ is rank $K$ and $Z_{J_l}$ injective because by assumption $J_l$ intersects every community, the space spanned by the columns of $\hat{Q}_{J_l}'$ is equal to the space spanned by the columns of $Z_{J_l}$. So we can write for each $l$, $\hat{Q}_{J_l}'=Z_{J_l}(Z^TZ)^{-1/2}S_l$ where $S_l\in \mathbb{R}^{K\times K}$ is invertible. Fix $l$ and $l'$. For all $i\in J_l\cap J_{l'}$, \[ \hat{Q}'_{i*}=(Z_{J_l}(Z^TZ)^{-1/2}S_l)_{i*} =Z_{i*}(Z^TZ)^{-1/2}S_l =(Z_{J_{l'}}(Z^TZ)^{-1/2}S_{l'})_{i*}= Z_{i*}(Z^TZ)^{-1/2}S_{l'}.\] Since by assumption $ J_l\cap J_{l'}$ intersects every community, we get $(Z^TZ)^{-1/2}S_l=(Z^TZ)^{-1/2}S_{l'}$ and hence $S_l=S_{l'}:=S$ for all $l, l'$. Finally the condition $(\hat{Q}')^T \hat{Q}'=I_K$ implies $S^T S=I_K$ so $\hat{Q}'=\hat{Q}O$ where $O\in \mathbb{R}^{K\times K}$ is orthogonal. The matrix $\hat{B}^{'(l)}$ solution of \eqref{prop1:eq} is uniquely determined by \[ ((\hat{Q}_{J_l}')^T\hat{Q}_{J_l}')^{-1}\hat{Q}_{J_l}'^T Z_{J_l}\Pi^{(l)}Z_{J_l}^T\hat{Q}_{J_l}'((\hat{Q}_{J_l}')^T\hat{Q}_{J_l}')^{-1}.\] This last expression can be rewritten as \[ O(\hat{Q}_{J_l}^T\hat{Q}_{J_l})^{-1}\hat{Q}_{J_l}^T Z_{J_l}\Pi^{(l)}Z_{J_l}^T\hat{Q}_{J_l}(\hat{Q}_{J_l})^T\hat{Q}_{J_l})^{-1}O^T= O\hat{B}^{(l)}O^T. \] So under the assumption of Proposition \ref{prop:olmf} the solutions of \eqref{olmf:miss_def} are unique up to an orthogonal transformation and the column span of $\hat{Q}$ is the same as the column span of $Z$.


\begin{remark}
The event ``for each $l, l'$ the sets $J_l\cap J_{l'}$ intersect all communities'' occurs with probability at least $1-O(KL^2/n^2)$ by replacing $J_l$ by $J_l\cap J_{l'}$ in Lemma \ref{lem:sizenk}.
\end{remark}

%
\section{Auxiliary Lemmas}
We first recall the standard Chernoff bound for sum of independent Bernoulli random variables.
\label{app:aux_lemmas}
\begin{lemma}[Chernoff bound]
\label{lem:chernoff}
Let $X=\sum_{i\leq n} X_i$ where $X_i\overset{\text{ind.}}{\sim} \mathcal{B}(\rho)$. Then \[ \prob(X\leq (1- \delta) n\rho)\leq e^{-n\rho \delta^2/2}\] and \[ \prob(X\geq (1+ \delta) n\rho)\leq e^{-n\rho \delta^2/3}\] for all $0<\delta <1$.
\end{lemma}
\begin{proof}
See \cite[Theorem 4.5 and Corollary 4.6]{random_algBook} .
\end{proof}

\begin{lemma}
\label{lem:sizenk}
Under the assumptions of Theorem \ref{thm:kpod}, with probability at least $1-O(KL/n^2)$, it holds for each $k=1,\ldots , K$ and $l=1,\ldots ,\leq L$ that
\[ \frac{\rho}{2}n_k \leq n_{k,J_l}\leq 2\rho n_k. \]
\end{lemma}
\begin{proof}
Recall that $n_{k,J_l}=\sum_{i\in \calC_k}\indic_{i\in J_l}$ is a sum of $n_k$ independent Bernoulli random variables with parameter $\rho$. By applying Lemma \ref{lem:chernoff} with $\delta=1/2$ we get \[ n_{k,J_l}\geq \frac{\rho}{2}n_k\] and \[n_{k,J_l} \leq 2n_k \] with probability at least $1-O(1/n^2)$, provided that $n_k\rho\geq C\log n$ for a constant $C$ large enough as assumed in Theorem \ref{thm:kpod}. The lemma follows from a  union bound.
\end{proof}

\begin{lemma}[Matrix Bernstein inequality]
\label{mat_bern}
Let $X_1, \ldots, X_n$ be a sequence of independent zero-mean random matrices of size $d_1\times d_2$. Suppose that $||X_i||\leq M$ almost surely, for all $i$. Then for all positive $t$, \[ \prob(||\sum_i X_i||\geq t)\leq (d_1+d_2)\exp\left(-\frac{t^2}{2\sigma ^2+2M/3t}\right)\] where $\sigma^2=\max( ||\sum_i \expec(X_iX_i^*)||,||\sum_i \expec(X_i^*X_i)||)$.
\end{lemma}
\begin{proof}
See \cite[Theorem 1.6]{user_friendly}
\end{proof}

\begin{lemma}
\label{bandeira}
Let $X$ be an $n \times n$ symmetric matrix whose entries $X_{ij}$ are independent centered random variables. Then there exists for any $0<\epsilon\leq 1/2$ a universal constant $c_{\epsilon}$ such that for every $t\geq 0$
\[ \prob(||X||\geq 2(1+\epsilon)\tilde{\sigma}+t)\leq \exp\left(-\frac{t^2}{\tilde{c}_\epsilon\tilde{\sigma}_{*}} \right)\] where $\tilde{\sigma}=\max_i \sqrt{\sum_j \expec(X_{ij}^2)}$ and $\tilde{\sigma}_{*}=\max_{i,j}\expec||X_{ij}||_\infty$.
\end{lemma}
\begin{proof}
See \cite[Corollary 3.12 and Remark 3.13 ]{bandeira2016}.
\end{proof}

\begin{lemma}
\label{lem:concAdj}
 Let $A\in \mathbb{R}^{n\times n}$ be an adjacency matrix generated by a $SBM(Z,\Pi)$. Denote $\lambda_K$ to be the $K$th largest singular value of $P=Z\Pi Z^T$. If $\lambda_K > 2||A-\expec(A)||$, then with probability at least $1-O(n^{-2})$ \[ ||\hat{U}-UO||_F\leq C\sqrt{K} \frac{\sqrt{\max(np_{max},\log n)}}{\lambda_K}\] where $\hat{U}$ is the matrix formed by the first $K$ left singular vectors of $A$, $U=Z(Z^TZ)^{1/2}$ and $O$ is the orthogonal matrix that aligns $\hat{U}$ and $U$.
\end{lemma}
%
%
\begin{proof}
By Remark 3.13 in \cite{bandeira2016} we get that $||A-\expec (A)|| \leq C \sqrt{\max(np_{max},\log n)}$ with probability at least $1-O(n^{-2})$. Moereover, since $\hat{U}$ and $UO$ are at most rank $K$ matrices we have \[ ||\hat{U}-UO||_F\leq  \sqrt{2K}||\hat{U}-UO||.\] Wedin's theorem (see \cite{wedin1972perturbation}) implies that \begin{equation}\label{lem:wedin1}
    ||\hat{U}-UO||\leq \frac{||A-\expec(A)||}{\delta}
\end{equation}  where $\delta:=| \lambda_K(A)-\lambda_{K+1}(\expec(A))|$ represents the spectral gap. By Weyl's inequality, \[|\lambda_{K+1}(\expec(A))-\lambda_{K+1}(Z\Pi Z^T)|\leq ||\expec(A)-Z\Pi Z^T ||.\]  Since $\expec(A)-Z\Pi Z^T$ is a diagonal matrix, its spectral norm is bounded by its largest coefficient that is bounded by $p_{max}$. Moreover since $\lambda_{K+1}(Z\Pi Z^T)=0$ we get $\lambda_{K+1}(\expec(A))\leq p_{max}$. The same argument can be used to show that $\lambda_{K}(\expec(A))\geq \lambda_K(Z\Pi Z^T)-p_{max}$. 

 Weyl's inequality also implies that \[ |\lambda_{K}(\expec(A))-\lambda_{K}(A)|\leq ||A-\expec (A)||. \] Thus \[\lambda_K(A)\geq \lambda_K(\expec(A))- ||A-\expec (A)||\geq \lambda_K(Z\Pi Z^T)-p_{max}-||A-\expec (A)|| \geq \frac{1}{2}\lambda_K(Z\Pi Z^T)-p_{max}.\] The last inequality follows from the assumption that $\lambda_K(Z\Pi Z^T)\geq 2||A-\expec (A)|| $. Since $p_{max}\leq \epsilon(n) \lambda_K(Z\Pi Z^T)$ where $\epsilon(n)\to 0$ when $n\to \infty$, $\lambda_K(A)\geq c \lambda_K(Z\Pi Z^T)$ and then $\delta \geq c \lambda_K(Z\Pi Z^T)$.
 Therefore the concentration bound stated at the beginning of the proof and \eqref{lem:wedin1} implies \[   ||\hat{U}-UO||_F\leq \sqrt{2K}\frac{||A-\expec(A)||}{\delta} \leq C\sqrt{K}\frac{\sqrt{\max(np_{max},\log n)}}{\lambda_K}\] with probability at least $1-O(n^{-2})$. 
\end{proof}

\begin{lemma}
\label{lem:eigval_min}
We have $\lambda_K(Z\Pi Z^T)\geq n_{min}\lambda_{K}(\Pi)$.
\end{lemma}

\begin{proof}
Let $\mu_1 \geq \mu_2 \geq \cdots \geq \mu_K$ be the $K$ \emph{non-zero} eigenvalues of $Z\Pi Z^T$. By the variational characterization of eigenvalues we have for all $k=1,\dots,K$
\begin{equation}\label{variat_char}
    \mu_k(Z\Pi Z^T)= \min_{V \subset G_{n-k+1}}\max_{\substack{x\in V\\ ||x||=1}} x^TZ\Pi Z^Tx
\end{equation} 
where $G_{n-k+1}$ denotes the set of $n-k+1$ dimensional subset of $\mathbb{R}^n$. Observe that $\ker(Z^T)=\Ima (Z)^\perp$.  An element $x\in \ker(Z^T)$ cannot be a solution because $Z\Pi Z^T$ is a rank $K$ matrix and thus $\mu_k(Z\Pi Z^T)\neq 0$, so for $k \leq K$ the optimization problem \eqref{variat_char} is equivalent to 
\begin{equation}\label{variat_char2}
    \mu_k(Z\Pi Z^T)= \min_{V \subset G_{n-k+1}}\max_{\substack{x\in V\cap \Ima (Z)\\ ||x||=1}} x^TZ\Pi Z^Tx.
\end{equation}  It implies in particular that any eigenvector of $Z\Pi Z^T$ associated with $\mu_k$ belongs to $\Ima (Z)$, so it has a block structure.  
Let $v$ be an eigenvector associated with $\mu_k(Z\Pi Z^T)$ for $1 \leq k \leq K$. Then $v=Zu$ where $u\in \mathbb{R}^{K}$. In particular $Z\Pi Z^Tv=Z w$ where $w=\Pi Z^TZu$. Thus \[\mu_k^2(Z\Pi Z^T) = ||Z\Pi Z^Tv||^2\geq n_{min}||w||^2\geq  n_{min}\lambda_K^2(\Pi)||Z^TZu||^2 \geq n_{min}^2\lambda_K^2(\Pi)||v||^2\] because the least singular value of $Z$ is $\sqrt{n_{min}}$. Clearly, this in particular implies that  $\lambda_K(Z\Pi Z^T)\geq n_{min}\lambda_{K}(\Pi)$.
\end{proof}

\section{Comparison between misclustering bound under MLSBM in the complete setting}
\label{app:miscl_bounds}
Here we compare existing bounds for the misclustering rate under the MLSBM in the complete data setting.
In order to simplify the comparison between the existing bounds, we will assume that $K$ is a constant, the communities are well balanced and $p_{max}^{(l)} \approx p_{max}$ for each $l$.
\begin{itemize}
    \item \textbf{Co-regularized spectral clustering}. 
    This algorithm was introduced by \cite{coregSC}. It is an intermediate fusion method that aims to find the best set of eigenvectors that simultaneously approximate  the set of eigenvectors associated with each individual layer. It was shown later by \cite{paul2020} that if $Lnp_{max}\geq C\log n$ and  $\Pi^{(l)}$ is full rank for all $l$, then with high probability (w.h.p) \[ r_{coreg} = O\left(\sqrt{\frac{\log n}{Lnp_{max}}}\right).\]

    \item \textbf{OLMF}. This estimator was discussed earlier in Section \ref{subsec:olmf_complete}. It was shown by \cite{paul2020} that if $np_{max}\geq C \log n$ and at least one of the matrices $\Pi^{(l)}$ is full rank then w.h.p. \[ r_{OLMF} = O\left(\frac{1}{\sqrt{np_{max}}}\max \set{1, \frac{(\log n)^{2+\epsilon}  \sqrt{\log L}}{L^{1/4}}}\right). \]
    
    \item \textbf{Sum of adjacency matrices.} It was shown by  \cite{paul2020} that if $Lnp_{max}\geq \log n$ and $\lambda_K(\sum_l \Pi^{(l)})\approx L p_{max}$ then w.h.p. \[ r_{sum} = O\left(\frac{\log n}{L np_{max}} \right).\]  \cite{Bhattacharyya2018SpectralCF} showed that if $Lnp_{max}\geq \log n$ and $\lambda_K(\sum_l \Pi^{(l)})\approx L \lambda_K(\Pi^{(1)})$ then w.h.p. \[ r_{sum}=O\left(\frac{1}{\sqrt{L np_{max}}}\right).\] The condition $Lnp_{max}\geq \log n$ is not stated in  \cite{Bhattacharyya2018SpectralCF} and is only assumed here for simplification. This last bound is better than the former in the sparse case when $Lnp_{max}\approx \log n$. But when $np_{max}\gg \log n^2$ the first bound is sharper.
    
    \item \textbf{Bias adjusted sum of the squared adjacency matrices.} 
    Sum of adjacency matrices performs badly when some layers are associative and other disassociative. Taking the sum of the square of adjacency matrices instead permits us to overcome this issue. However the diagonal entries of these squared matrices introduce bias, so they are often removed. More involved debiasing strategies have also been considered by \cite{hetePCA} and \cite{kmeans_relax}. 
    Assume $L=O(n)$. In the sparse case when $\sqrt{L}np_{max}\geq C\sqrt{\log n}$ and $np_{max}=O(1)$, \cite{lei2020tail}  showed that w.h.p. \[ r_{sq}=O\left(\frac{1}{n}+\frac{\log n}{L (np_{max})^2}\right).\]
    If $np_{max}\geq C\sqrt{\log n}$ they showed that w.h.p. \[ r_{sq}=O\left(\frac{\log n}{\sqrt{L}np_{max}}\right).\]
    This method was also analyzed by \cite{bhattacharyya2020general}. They showed that if $Lnp_{max}\geq C\log n$ then w.h.p. \[ r_{sq}=O\left(\frac{1}{(Lnp_{max})^{1/2}}\right).\]
\end{itemize}


\section{Additional experiments}

We added two alternative algorithms in our experiments. \begin{itemize}
    \item \lapla: the matrix $A=L^{-1}\sum_l A^{(l)}\odot \Omega^{(l)}$ is replaced by its normalized Laplacian $\mathcal{L}:=D^{-1/2}AD^{-1/2}$ where $D$ is a diagonal matrix such that $D_{ii}=\sum_j A_{ij}$. The experiments show that using this normalization improves the misclustering rate only in regimes where the sum of adjacency matrices gives good results.
    
    \item \aggrk: it is a generalization of the aggregate spectral kernel method introduced in \cite{paul2020}. For each layer $l$ we compute $\hat{U}^{(l)}$ as in Algorithm \ref{alg:k-pod}, compute the top $K$ singular vectors of $\sum_l \hat{U}^{(l)}(\hat{U}^{(l)})^T$ and then perform $k-$means on the rows of the matrix formed by these singular vectors. This method performs slightly better than \kpod  \ in our experiments.
\end{itemize}
%
Figures \ref{plot_rho_bis}, \ref{plot_L_bis} and \ref{plot_n_bis} correspond to simulations run for the same generative model as described in Section \ref{subsec:num_syn_exps}. Figure \ref{plot_k_unbalanced} corresponds to simulations run for three unbalanced communities generated from a multinomial law with parameters $(1/6,1/6,2/3)$. The diagonal (resp. off-diagonal) entries of the connectivity matrices are equal to $0.2$ ( resp. $0.1$). Figure \ref{plot_rho_bis} shows that when $\rho$ and $L$ are small, \sumadjz \, seems to be the best method. However, when $L$ is much larger (and $\rho$ small), \olmfm \, performs best. When the number of layers increase, algorithms based on early or intermediate fusion (\sumadjz,\sumadji , \olmfm) outperform algorithms based on final aggregation (\kpod, \aggrk) as shown in Figure \ref{plot_L_bis}. Final aggregation methods (\kpod, \aggrk) are more sensitive to the number of nodes than  other methods, see Figure \ref{plot_n_bis}.  When the community sizes are unbalanced, we need a stronger separation between community to recover the small community but the relative performance of the proposed algorithms seem to be similar as shown in Figure \ref{plot_k_unbalanced}.

 \begin{figure}[!ht]
  \center
   \includegraphics[scale=0.6]{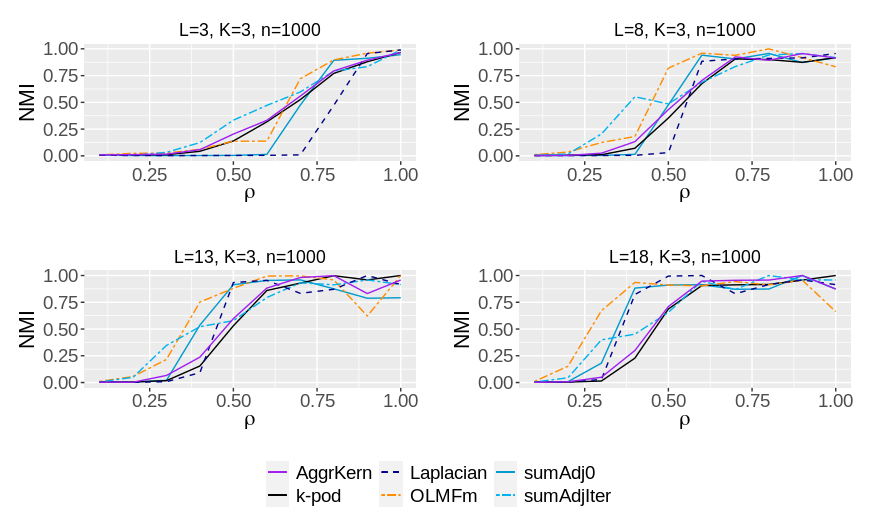}
   \caption{ NMI vs $\rho$ for different values of $L$} 
   \label{plot_rho_bis}
\end{figure}

 \begin{figure}[!ht]
 \center
   \includegraphics[scale=0.6]{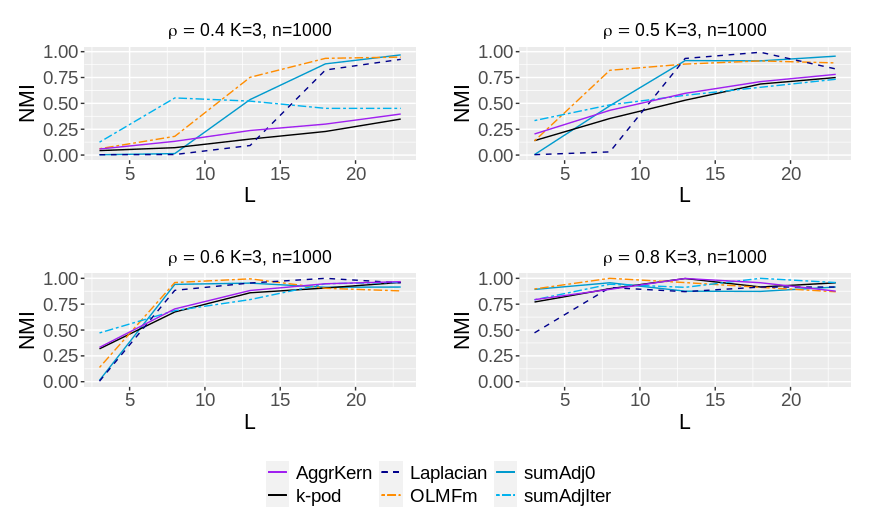}
     \caption{ NMI vs $L$ for different values of $\rho$
     }
     \label{plot_L_bis} 
\end{figure}

 \begin{figure}[!ht]
 \center
   \includegraphics[scale=0.6]{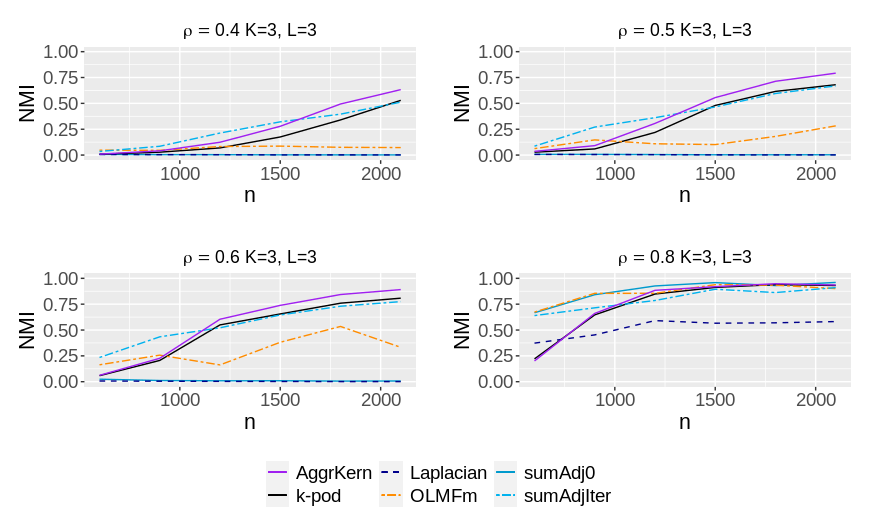}
     \caption{  NMI vs $n$ for different values of $\rho$
     }
     \label{plot_n_bis}
\end{figure}

 \begin{figure}[!ht]
 \center
   \includegraphics[scale=0.6]{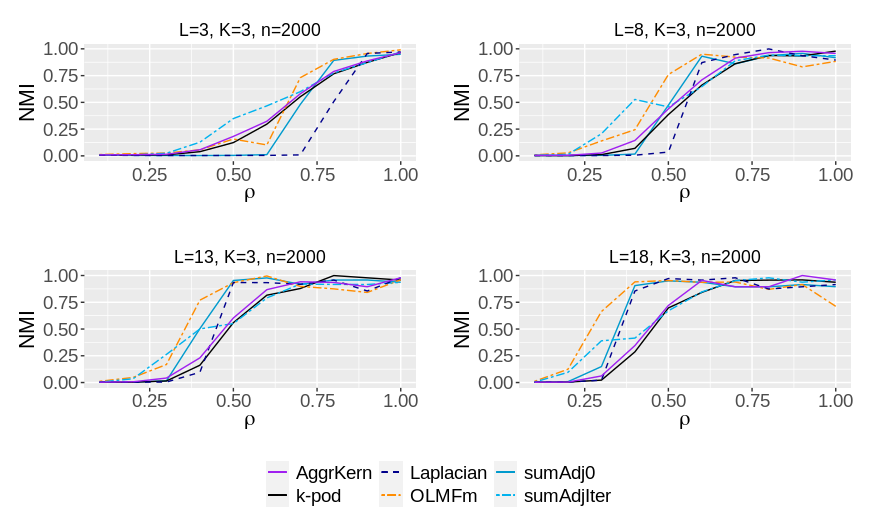}
     \caption{  NMI vs $\rho$ for different values of $L$ with unequal sized communities 
     }
     \label{plot_k_unbalanced}
\end{figure}

\end{document}